\documentclass{vldb}

\pdfoutput=1

\newcommand{\ARXIV} % Ben: comment out to make paper shorter (VLDB)

\usepackage{graphicx}
\usepackage{array}
\usepackage{balance}  % for  \balance command ON LAST PAGE  (only there!)
\usepackage{amsmath,amssymb,bm}%,amsthm}
\usepackage{soul,color}
\usepackage{algorithm,algpseudocode}
\usepackage{url}
\usepackage[shortlabels]{enumitem}
\usepackage{xspace}
\usepackage{multirow}
\usepackage[tableposition=top,labelfont=bf,textfont=bf]{caption}

\setlist[itemize]{itemsep=0ex, topsep=3pt}
\setlist[enumerate]{itemsep=0ex, topsep=3pt}

\graphicspath{{./Figures/}}

\newtheorem{theorem}{Theorem}

\newtheorem{lemma}{Lemma}
\newtheorem{definition}{Definition}
\newtheorem{example}{Example}
\newtheorem{remark}{Remark}

\algnewcommand\algorithmicinput{\textbf{Input:}}
\algnewcommand\Input{\item[\algorithmicinput]}
\algnewcommand\algorithmicoutput{\textbf{Output:}}
\algnewcommand\Output{\item[\algorithmicoutput]}

\newcommand{\Break}{\State \textbf{break} }

\renewcommand\vec[1]{\bm{#1}}
\newcommand{\term}[1]{\emph{#1}}
\newcommand{\alg}{OASIS\xspace}
\newcommand{\eg}{\emph{e.g.,}\xspace}
\newcommand{\ie}{\emph{i.e.,}\xspace}

\newcommand{\cf}{\emph{cf.}\xspace}
\newcommand{\etal}{\emph{et~al.}\xspace}
\newcommand{\reals}{\ensuremath{\mathbb{R}}\xspace}

\newcommand{\cR}{\ensuremath{\mathcal{R}}\xspace}
\newcommand{\cD}{\ensuremath{\mathcal{D}}\xspace}
\newcommand{\cZ}{\ensuremath{\mathcal{Z}}\xspace}
\newcommand{\TP}{\ensuremath{\mathrm{TP}}\xspace}
\newcommand{\FP}{\ensuremath{\mathrm{FP}}\xspace}
\newcommand{\FN}{\ensuremath{\mathrm{FN}}\xspace}
\newcommand{\oracle}{\ensuremath{\mathtt{Oracle}}\xspace}
\newcommand{\myparagraph}[1]{\vspace{0.5em}\noindent\textbf{#1}\xspace}

\begin{document}

% ****************** TITLE ****************************************

\title{In Search of an Entity Resolution OASIS: \\ Optimal Asymptotic Sequential Importance Sampling}

% Italics in title
%\title{{\ttlit Italic}}

% ****************** AUTHORS **************************************


\numberofauthors{1} %  in this sample file, there are a *total*

\author{
% 1st. author
\alignauthor
Neil G.\ Marchant and Benjamin I.\ P.\ Rubinstein\\
       \affaddr{School of Computing and Information Systems}\\
       \affaddr{University of Melbourne, Australia}\\
       \email{\string{nmarchant, brubinstein\string}@unimelb.edu.au}
% 2nd. author
%\alignauthor
%Benjamin I.\ P.\ Rubinstein\\
%       \affaddr{Department of Computing and Information Systems}\\
%       \affaddr{University of Melbourne, Australia}\\
%       \email{brubinstein@unimelb.edu.au}
}

\maketitle

\begin{abstract}
Entity resolution (ER) presents unique challenges for evaluation methodology. While crowdsourcing platforms acquire ground truth, sound approaches to sampling must drive labelling efforts. In ER, extreme class imbalance between matching and non-matching records can lead to enormous labelling requirements when seeking statistically consistent estimates for rigorous evaluation. This paper addresses this important challenge with the \alg algorithm: a sampler and F-measure estimator for ER evaluation. \alg draws samples from a (biased) instrumental distribution, chosen to ensure estimators with optimal asymptotic variance. As new labels are collected \alg updates this instrumental distribution via a Bayesian latent variable model of the annotator oracle, to quickly focus on unlabelled items providing more information. We prove that resulting estimates of F-measure, precision, recall converge to the true population values. Thorough comparisons of sampling methods on a variety of ER datasets demonstrate significant labelling reductions of up to 83\% without loss to estimate accuracy.
\end{abstract}

\section{Introduction}
% Introduce ER
% It's well motivated
% Used in...

The very circumstances that give rise to entity resolution (ER) systems---lack of shared keys
between data sources, noisy\slash missing features, heterogeneous distributions---explain the
critical role of evaluation in the ER pipeline~\cite{christen2007quality}. Production systems
rarely achieve near-perfect
precision and recall due to these many inherent ambiguities, and when they do, even minute
increases to error rates can lead to poor user experience~\cite{negahban2012scaling},
lost business~\cite{verykios2000accuracy}, or erroneous diagnoses and public health
planning~\cite{harron2014evaluating}. 
It is thus vital that ER systems are evaluated in a statistically sound manner so as to 
capture the true accuracy of entity resolution.
This paper addresses this challenge with the development of an algorithm based on adaptive
importance sampling, which we call `\alg'.

While crowdsourcing platforms provide inexpensive provisioning of 
annotations, sampling items for labelling must proceed carefully. A key challenge in ER is the
inherent imbalance between matching and non-matching records which
can be as high as $1:n$ when matching two sources of $n$ records (\eg reaching the millions). 
Researchers leverage several existing practices to evaluate such an ER system:
(i) Label samples drawn from all
candidate matches uniformly
at random (\eg record pairs in two-source integration): while yielding unbiased estimates, this can take
thousands of samples before finding one match-labelled sample, and many tens of thousands
of labels before estimates converge. (ii) Balance inefficient passive sampling with
cheap crowdsourcing resources: while crowdsourcing facilitates ER evaluation, large
nonstationary datasets require constant refresh and can quickly drive costs back up. (iii) Exploit blocking schemes or search facilities to 
reduce non-match numbers: such filtering injects hidden bias into estimates. 

%Even more sophisticated approaches tend not to offer much improvement, either failing to perform in imbalanced settings, or simply due to a suboptimal design. 

By contrast, \alg offers a principled alternative to evaluating F-measure, precision,
recall---robust measures under imbalance---given an ER system's set of output similarity
scores. \alg forms an instrumental distribution 
from which it samples record pairs non-uniformly, minimising the estimator's asymptotic variance.
This instrumental distribution is based on estimates of latent truth due to a simple Bayesian model, and is updated iteratively. %as labels are collected using standard posterior updates and recalculation of minimal asymptotic variance.
By stratifying the pool of record pairs by similarity score, \alg transfers
performance estimates and samples fewer points. By ensuring our sampler may (with non-zero
probability) sample any stratum, we manage the explore-exploit trade-off, admitting
guarantees of statistical consistency: our estimates of F-measure, precision,
recall converge to the true population parameters with high probability.

%The characteristics of \alg that adapts to observed labels, stratifies for efficiency,
%leverages auxiliary information from similarity scores, corrects bias \& guarantees consistency,
%executes quickly, and focusses on estimating the non-linear F-measure, each set \alg
%apart from existing work in most cases, and together yield an approach to ER evaluation
%that can use orders-of-magnitude fewer labels. This is borne out in thorough comparisons
%of baselines across six datasets of varying sizes and class imbalance (up to over 1:3000).

The unique characteristics of \alg together yield a rigorous approach to ER evaluation
that can use orders-of-magnitude fewer labels. This is borne out in thorough comparisons
of baselines across six datasets of varying sizes and class imbalance (up to over 1:3000).

%A novel feature of OASIS over the current state-of-the-art method is its adaptivity. OASIS learns from the labels it receives in order to decide where in the dataset it should focus sampling so as to obtain a precise performance estimate. This can have a marked effect on reducing the overall labelling requirements, because it permits OASIS to approach the optimal sampling strategy.

%In order to demonstrate the effectiveness of OASIS over alternative methods, we conduct comprehensive experiments on several ER evaluation datasets. The results demonstrate that OASIS almost invariably outperforms the state-of-the-art, achieve a 75\% reduction in the labelling requirements on one evaluation task with a class imbalance ratio (matches to non-matches) of 1:3400.

% Also point out that evaluation is rarely done once. May want to see how the performance varies whenever tweaks are made to the system. If dataset shift is an issue, then it may be necessary to periodically re-evaluate a system in deployment.

\myparagraph{Contributions.} 
1) The novel \alg algorithm for efficient evaluation of ER based on adaptive importance sampling. This algorithm has been released as an open-source Python package at 
https://git.io/OASIS ; \\[0.25em]
2) Theoretical guarantee that \alg yields statistically consistent estimates, made challenging by the non-independence of the samples and the non-linearity of the F-measure; and \\[0.25em]
%2) With an instrumental distribution that optimises estimates' asymptotic variance, theoretical guarantees that \alg yields statistically-consistent estimates, made challenging by dependent labels, F-measure non-linearity; and \\[0.25em]
3) A comprehensive experimental comparison of \alg with existing state-of-the-art algorithms demonstrating superior performance \eg 83\% reduction in labelling requirements under a class imbalance of 1:3000.

\section{Background}

Motivated by the challenges of accurate but efficient evaluation of 
ER, we begin by reviewing the key features of ER.

\subsection{Entity resolution}
\label{sec:ER-system}

\begin{definition}[ER problem]
Let $\cD_1$ and $\cD_2$ denote two databases, each containing
a finite number of records $n_1, n_2$ representing underlying entities; and let
fixed, unknown relation $\cR \subseteq \cD_1 \times \cD_2$ describe
the matching records across the databases, \ie pairs of records representing
the same entity. The \term{entity resolution problem} is to approximate $\cR$ with
a predicted relation $\hat{\cR} \subseteq \cD_1 \times \cD_2$.
\end{definition}

\begin{remark}
For simplicity we focus on two-source ER, however
our algorithms and theoretical results apply equally well to multi-source
ER on relations over larger product \linebreak[4] spaces, and deduplicating a single source. %it can be easily generalised to a single database by setting $Z = D^{(1)} \times D^{(1)}$ (excluding obvious duplicates), or multiple databases by running ER on all database pairs: e.g.\ $D^{(1)} \times D^{(2)}$, $D^{(2)} \times D^{(3)}$ and $D^{(1)} \times D^{(3)}$.
\end{remark}

An abundant literature describes the typical ER pipeline:
\emph{preparation} amortising record canonicalisation; \emph{blocking} for reducing
pair comparisons through a linear database scan; \emph{scoring}, the most
expensive stage, in which pair attributes are compared and summarised
in similarity scores; and \linebreak[4] \emph{matching} where sufficiently high-scoring pairs
are used to construct $\hat{\cR}$. Further normalisation pre- or post-linkage
such as schema matching or record merging, while non-core, are important
also. We refer the interested reader to review
articles~\cite{winkler_overview_2006,christen_data_2012,getoor_entity_2013}
and the references therein. 
%Recent trends: active learning~\cite{arasu_active_2010,bellare_active_2012}, learning custom string distance measures~\cite{bilenko_adaptive_2005}, crowdsourcing \cite{wang_crowder_2012}

\subsubsection{Similarity scores}
%We enumerate the records of database $\cD_i$ as
%$\{r_1^{(i)}, \ldots, r_{n_i}^{(i)}\}$.
ER is often cast as a binary classification problem on the set of record pairs
$\cZ = \cD_1 \times \cD_2$. A pair $z \in \cZ$
has true Boolean label 1 if a ``match'', that is $z\in\cR$, and label 0 if a
``non-match'', that is $z\notin\cR$. In this work, we leverage the
similarity scores produced in typical ER pipelines: 

\begin{definition}
A \term{similarity score} $s(z)\in\reals$ quantifies the level of similarity
that a given pair $z\in\cZ$ exhibits, \ie the predicted confidence of a match.
\end{definition}

Similarity scores originate from a variety of sources. The scoring 
phase of typical ER pipelines combine attribute-level dis/similarity measures
\eg edit distance, Jaccard distance, absolute deviation, etc., into similarity scores.
The combination itself is often produced by hand-coded rules or supervised classification, fit to a 
training set of known non/matches. 
Unlike in evaluation, data used for training need not be representative: heuristically-compiled training sets may be used when learning discriminative models.
Any \emph{confidence-based classifier}, \eg the 
support vector machine, or  
\emph{probabilistic classifier}, \eg logistic regression or probability trees, produces legitimate similarity scores. Scores from probabilistic classifiers may or may not be \term{calibrated}:

\begin{definition}\label{def:calibration}
	A scoring function $s(\cdot)$ is \term{calibrated} if, of all the record pairs mapping
	to $s(z)=\rho\in[0,1]$, approximately
$100\times\rho$ percent are truly matching. For example, 60\% of pairs with a score of 0.6 should 
be matches.
\end{definition}

%Even hand-tuned decision scores such as cosine similarity
%on two string attributes summed, produces a legitimate similarity score function.

%likelihood that a given record pair $z \in Z$ is a ``match''. \hl{But they don't need to be probabilities.} Almost all classifiers are able to produce similarity scores. For probabilistic classifiers (e.g.\ logistic regression, decision trees) the probability estimates of $p(1|z)$ can be used as the similarity scores. In other cases (e.g.\ SVM) it may be more natural to use the distance between the decision boundary/surface and the record pair $z$.

%We note that similarity scores can also be obtained when running \emph{deterministic ER}, that is ER where the classifier is not learned from data, but constructed heuristically by a domain expert. \hl{What are the scores in this case?}

\subsection{Evaluation measures for ER}
\label{sec:performance-measures}

All ER evaluation methods produce statistics that summarise the types of errors made in approximating \cR with $\hat{\cR}$. Arguably the most popular among these statistics is the pairwise F-measure which we focus on in this work. The F-measure is particularly well suited to ER, unlike accuracy for example, as its invariance to true negatives makes it more robust to class imbalance. 
The F-measure is a weighted harmonic mean of precision and recall; and
in terms of Type I and Type II errors, the statistic on $T$ labels is
\begin{eqnarray}
	F_{\alpha,T} &=& \frac{\TP}{\alpha(\TP + \FP) + (1 - \alpha)(\TP + \FN)}\enspace,
	\label{eqn:F-measure-finite}
\end{eqnarray}
where $\alpha \in [0, 1]$ is a weight parameter; TP, FP, FN are true
positive, false positive, false negative counts respectively.
\begin{equation*}
	\TP = \sum_{t=1}^T \ell_t \hat{\ell}_t\enspace,\enspace
	\FP = \sum_{t=1}^T (1-\ell_t) \hat{\ell}_t \enspace,\enspace
	\FN = \sum_{t=1}^T \ell_t (1-\hat{\ell}_t) \enspace,
\end{equation*}
where $z_1,\ldots,z_T\sim p$ are query pairs sampled i.i.d from some underlying distribution
$p$ of interest on \cZ such as the uniform distribution; the $\ell_t$ denote ground truth labels recording
(possibly noisy) membership of $z_t$ within \cR; and $\hat{\ell}_t$ indicates $z_t\in\hat{\cR}$. 
When $\alpha = 1$, $F_{\alpha,T}$ reduces to precision, $\alpha = 0$ produces recall, and $\alpha = 1/2$ yields the balanced F-measure, with equal importance on precision and recall.\footnote{The relationship to the $\beta$-parametrisation is $\alpha = 1/(1 + \beta^2)$.}

Our goal will be to estimate the asymptotic limit of $F_{\alpha,T}$ as label budget $T\to\infty$.
For finite pools \cZ this corresponds to labelling of all record pairs with sufficient repetition
to account for (any) noise in the ground truth labels $\ell_t$.

%As stated, the target $F_\alpha$ is calculated on the entire pool \cZ. In practice
%it is estimated on a much smaller sample of labelled non/matching pairs.

\begin{remark}
The pairwise F-measure is termed ``pairwise" to highlight the application
of the measure to record pairs. Pairwise measures work well when there are
only a few records across the databases which correspond to a particular
entity. In such cases one should not use accuracy, due to significant class
imbalance (\cf Section~\ref{sec:problem-formulation}). For cases where most entities
have many matching records, one may leverage transitivity constraints while
looking to cluster-based measures for
evaluation~\cite{menestrina_evaluating_2010}. 
See \cite{barnes_practioners_2015} for a summary on evaluation.
\end{remark}

\section{Problem formulation}
\label{sec:problem-formulation}
Suppose we are faced with the task of evaluating an ER system as described in the previous section. 
Given that we do not know \cR, how can we efficiently leverage labelling resources to estimate the pairwise
F-measure?

\begin{definition}[Efficient evaluation problem]
\label{def:eval-problem}
~\\Consider evaluating a predicted ER $\hat{\cR}\subseteq\cZ=\cD_1\times \cD_2$, equivalently represented by predicted labels $\hat{\ell}(z)=\mathbf{1}\big[z\in\hat{\cR}\big]$ for $z\in\cZ$. We are given access to:
\begin{itemize}
	\item a pool\footnote{We introduce the pool $P$ for flexibility. It can be taken to be the entire $\cZ$, or a proper subset for efficiency.} $P\subseteq\cZ$ of record pairs, \eg $P=\cZ$;
%	\item a predicted ER $\hat{\cR}\subseteq\cZ=\cD_1\times \cD_2$, 
%		equivalently represented by predicted labels $\hat{\ell}_z=\hat{\ell}(z)=\mathbf{1}\big[z\in\hat{\cR}\big]$ for $z\in\cZ$; and
	\item a similarity scoring function $s:\cZ\to\reals$; and 
	\item a randomised labelling $\oracle: \mathcal{Z} \to \{0,1\}$, which returns labels $\ell(z) \sim \oracle(z)$ indicating membership in \cR. The oracle's response distribution is parametrised by \term{oracle probabilities} $p(1|z) = \Pr[\oracle(z) = 1]$.
\end{itemize}
%To facilitate the consideration of randomised oracles (e.g.\ in crowdsourcing), we permit the returned labels to be noisy and define \term{oracle probabilities} $p(1|z) = \Pr[\oracle(z) = 1]$. 

With this setup, the \term{efficient evaluation problem} is to devise an estimation procedure for $F_\alpha$, which samples record pairs $z_1,\ldots,z_T\in P$ and makes use of the corresponding labels provided by the oracle. We adopt integer index notation on $\hat{\ell}, \ell$ and $s$ to denote their values at the $t$-th query; \eg $\hat{\ell}_t = \hat{\ell}(z_t)$ for query $z_t$. 

Solutions should produce estimates $\hat{F}_{\alpha,T}$ exhibiting:
\begin{enumerate}[(i)]
	\item \textbf{consistency}: convergence in probability to the true value $F_\alpha$ on pool $P$ with respect to underlying distribution $p$
		\begin{equation}
			F_\alpha = \lim_{T\to\infty} F_{\alpha, T}\enspace;\enspace\mbox{and}\label{eqn:F-measure}
			\vspace{-1ex}
		\end{equation}
	\item \textbf{minimal variance}: vary minimally about $F_\alpha$.
\end{enumerate}
\end{definition}

In other words, solutions should produce precise estimates whilst minimising queries to the oracle, since it is assumed that queries come at a high cost. Computational efficiency of the estimation procedure is not a direct concern, so long as the response time of the oracle dominates (typically of order seconds in a crowdsourced setting).

ER poses unique challenges for efficient evaluation.

%For notational convenience, we assign a unique index to each of the pairs in $P$ --- i.e.\ we treat $P$ as an ordered set $\{z_1, z_2, \ldots, z_N\}$. Then, for each pair in $P$ we use the ER system to generate a set of similarity scores $S = \{s_i = s(z_i)\}_{i = 1:N}$ and a set of predicted labels $\hat{L} = \{\hat{\ell}_i = \hat{\ell}(z_i)\}_{i = 1:N}$.

%\begin{remark}
%	For simplicity, we assume that the labels provided by the oracle are free from noise. However, in practice, labels provided by humans are likely to be noisy observations of the truth. This can be dealt with by using crowd-sourced truth-finding techniques \hl{(references)}. Also, \hl{our method can probably handle} noise without these techniques by assuming the distribution $p(\ell)$ is not sharp?
%\end{remark}

%\subsection{Problem challenges}\label{sec:challenges}

\myparagraph{Challenge: Extreme class imbalance.}
The inherent class imbalance in ER presents a challenge for estimation
of F-measure. For deduped databases $\cD_1,\cD_2$, the minimum possible class
imbalance occurs when both DBs contain $n$ records and there is a matching record in $\cD_1$ for every record in $\cD_2$. In this case, the \term{class imbalance ratio} (ratio of non-matches to matches) is
$n - 1$.
This is problematic for passive (uniform i.i.d.) sampling even for modest-sized databases, since $\mathcal{O}(n)$ expected pairs would be
sampled for every match found. As $F_\alpha$ depends only on matches (both
predicted and true), many queries to the oracle would be wasted on labels
that don't contribute. The problem becomes one of searching for an oasis within
a desert when $n \sim 10^6$ or more.

\myparagraph{Approach: Biased sampling.}
\label{sec:biased-sampling}
One response to class imbalance is \term{biased sampling}, that is, sampling from a population 
(or space more generally) in a way that systematically differs from the underlying distribution~\cite[Chapter~5]{rubinstein_simulation_2007}. Biased sampling methods have found broad application in areas as diverse as survey methodology,
Monte Carlo simulations, and active learning, to name a few. They work by leveraging known information about the system---here the similarity scores and
the pool of record pairs---to obtain more precise estimates using fewer samples. One of the most effective biased sampling methods is \term{importance sampling (IS)}, which we illustrate below:
%In particular \term{importance sampling (IS)} can achieve substantial variance reduction
%when estimating rare events.

\begin{example}
Consider a random variable $X$ with probability density $p(x)$ and consider the
estimation of parameter $\theta = \operatorname{E}[f(X)]$. The standard (passive)
approach draws an i.i.d. sample from $p$ and uses the Monte Carlo estimator
$\hat{\theta} = \frac{1}{T} \sum_{i = 1}^{T}f(x_i)$. Importance sampling, by contrast,
draws from an \term{instrumental distribution} denoted by $q$. Even though the sample
from $q$ is biased (i.e.\ not drawn from $p$), an unbiased estimate of $\theta$ can be obtained by using the bias-corrected
estimator $\hat{\theta}^{\mathrm{IS}} = \frac{1}{T} \sum_{i = 1}^{T} \frac{p(x_i)}{q(x_i)} f(x_i)$.
\end{example}

An important consideration when conducting IS is the choice of instrumental distribution, $q$. If $q$ is poorly selected, the resulting estimator may perform worse than passive sampling. If on the other hand, $q$ is selected judiciously, so that it concentrates on the ``important'' values of $X$, significant efficiency dividends will follow. %A common approach is to choose $q$ so as to minimise the variance of the IS estimator:
%\begin{equation}
%	q^\star = \arg \underset{q}{\min} \ \mathrm{Var}(\hat{\theta}^{\mathrm{IS}}[q]).
%	\label{eqn:IS-variance-minimisation}
%\end{equation}
%Even if this minimisation problem can be solved, the problem may not end here, because $q^\star$ often depends on quantities which are unknown. For example, it may depend on $\theta$, the very thing we are trying to estimate---a chicken and egg problem. Issues such as these lead to work on adaptive importance sampling (AIS)~\cite{}. In AIS, information gained from past samples is used to iteratively approach the optimal $q$.

\section{A New Algorithm: \alg}
This section develops our new algorithm for evaluating ER---\term{Optimal Asymptotic Sequential
Importance Sampling (\alg)}. In designing an adaptive/sequential importance sampler (AIS),
we proceed in two stages: (i) choosing an appropriate instrumental distribution to optimise asymptotic variance of the estimator,
see Section~\ref{sec:alg-selecting-inst-dist}; and (ii) deriving an appropriate update rule
and initialisation process for the instrumental distribution, now restricted to score strata, see Sections~\ref{sec:alg-class-probs} and~\ref{sec:alg-initialisation}.
Section~\ref{sec:alg-finish} brings all of the components of \alg together, presenting
the algorithm in its entirety. Section~\ref{sec:consistency-analysis} presents a thorough
theoretical analysis of \alg.

\subsection{Selecting the instrumental distribution}
\label{sec:alg-selecting-inst-dist}
We begin by defining an estimator for the F-measure which corrects for the bias of AIS.
It is based on the standard estimator of Eqn.~\eqref{eqn:F-measure-finite}, with the addition of importance weights.
\begin{definition}
	Let $\{x_t = (z_t, \ell_t)\}_{t = 1}^T$ be a sequence of record pairs and labels, where the
	$t$-th record pair in the sequence is drawn from pool $P$ according to an instrumental distribution
	$q_t$, which may depend on the previously sampled items $\vec{x}_{1:t-1} = \{x_1, \ldots, x_{t-1}\}$
	and labels $\ell_t\sim\oracle(z_t)$. Then the \emph{AIS estimator for the F-measure} is given by
	\begin{equation}
	\hat{F}_{\alpha}^{\mathrm{AIS}} = \frac{\sum_{t = 1}^{T} w_t \ell_t \hat{\ell}_t}{\alpha \sum_{t = 1}^{T} w_t \hat{\ell}_t + (1-\alpha) \sum_{t = 1}^{T} w_t \ell_t}\enspace,
	\label{eqn:AIS-F-estimator}
	\end{equation}
	where $w_t = p(z_t)/q_t(z_t)$ is the importance weight associated with the $t$-th item, and $p$ denotes any underlying distribution on the record pairs from which the target
	$F_\alpha$ is defined.
\end{definition}

This definition assumes that the record pairs are drawn from an, as yet, unspecified sequence of
instrumental distributions $\{q_t\}_{t=1}^T$. It is important that these instrumental distributions
are selected carefully, so as to maximise the sampling efficiency. Later, we justify the choice of $\hat{F}_\alpha^{\mathrm{AIS}}$ by proving that it is consistent for $F_\alpha$ (\cf Theorem~\ref{thm:consistency-F-alg}).

\begin{remark}
\label{rem:}
In ER we take: $P\subseteq\cZ$ typically a DB product space $\cD_1\times\cD_2$ which is finite (but possibly massive);
and the $p$ through which $F_{\alpha,T}$ is most naturally defined is the uniform distribution on $P$ \ie 
placing uniform mass $1/N$ where $N=|P|$.
However \alg and its analysis actually hold more generally: pools $P$ of \emph{instances} that could be
uncountably infinite in size; and arbitrary marginal distributions $p$ on $P$.
\end{remark}

%The term $\sum_{t=1}^{T} w_t \ell_t \hat{\ell}_t$ is the weighted number of TP, the term $\sum_{t=1}^{T} w_t \hat{\ell}_t$ is the weighted number of predicted positives (TP + FP) and the term $\sum_{t=1}^{T} w_t \ell_t$ is the weighted number of actual positives (TP + FN).

\subsubsection{Variance minimisation}
A common approach for instrumental distribution design is based on the principle of variance
minimisation~\cite{rubinstein_simulation_2007}. In the ideal case, a single
instrumental distribution (for all $t$) is selected that minimises the variance of the estimator:
\begin{equation}
q^\star \in \arg \underset{q}{\min} \ \mathrm{Var}(\hat{F}_\alpha^{\mathrm{AIS}}[q])\enspace.
\label{eqn:IS-variance-minimisation}
\end{equation}

This optimisation problem is difficult to solve analytically, in part due to the intractability
of the variance term. However, by replacing variance with the \emph{asymptotic} variance (taking $T\to\infty$),
a solution is obtained as
\begin{equation}
\begin{split}
q^\star(z) & \propto p(z) \left[ (1- \alpha) (1 -  \hat{\ell}(z)) F_{\alpha} \sqrt{p(1|z)} \right. \\
& \left. \;\;\;\; + \hat{\ell}(z) \sqrt{\alpha^2 F_\alpha^2 (1 - p(1|z)) + (1-F_\alpha)^2 p(1|z)} \right], \label{eqn:optimal-inst-dist}
\end{split}
\end{equation}
where $p(z)$ is the underlying distribution on $P$ (see Remark~\ref{rem:}) and $p(1|z)$ is the oracle probability (see Definition~\ref{def:eval-problem}).  
The proof of this result is given in \cite{sawade_active_2010}. We call $q^\star(z)$ the \term{asymptotically optimal
instrumental distribution}, owing to its relationship with asymptotic minimal variance.

\subsubsection{Motivation for adaptive sampling}
Close examination of~\eqref{eqn:optimal-inst-dist} reveals that the asymptotically optimal
instrumental distribution depends on the true F-measure $F_\alpha$ and true oracle probabilities
$p(1|z)$, both of which are unknown a priori. This implies
that an adaptive procedure is well-suited to this problem: we estimate $q^\star$ at iteration $t$
using \emph{estimates} of $F_\alpha$ and $p(1|z)$, which themselves are based on the previously
sampled record pairs and labels $\vec{x}_{1:t-1}$. As the sampling progresses and labels are
collected, the estimates of $F_\alpha$ and $p(1|z)$ should approach their true values, and
$q_t^\star$ should in turn approach $q^\star$.

In order to implement this adaptive procedure, we must devise a way of iteratively estimating
$F_\alpha$ and $p(1|z)$. There is a natural approach for $F_\alpha$: we simply use
$\hat{F}_{\alpha}^{\mathrm{AIS}}$ at the current iteration. However, the oracle probabilities
present more of a difficulty. We outline one approach in Section~\ref{sec:alg-class-probs}.

\subsubsection{Exploration vs. exploitation}
In the subsequent analysis of \alg (\cf Section~\ref{sec:consistency-analysis}), we show that the
asymptotically optimal instrumental distribution given in Eqn.~\eqref{eqn:optimal-inst-dist}
does not guarantee consistency (convergence in probability). This is because it permits zero weight to be placed on some items, meaning that parts of the pool may never be explored. Consequently, we propose to replace $q^\star$ by an $\varepsilon$-greedy distribution
\begin{equation}
q(z) = \varepsilon \cdot p(z) + (1 - \varepsilon) \cdot q^\star(z)\enspace,
\label{eqn:epsilon-greedy-inst-dist}
\end{equation}
where $0 < \varepsilon \leq 1$. For $\varepsilon$ close to 0, the sampling approaches optimality (it exploits), whereas for $\varepsilon$ close to 1, the sampling approaches passivity (it explores). This bears resemblance to explore-exploit trade-offs commonly encountered in online decision making (\eg multi-armed bandits)~\cite{cesa-bianchi_prediction_2006}.

\subsection{Estimating the oracle probabilities}
\label{sec:alg-class-probs}
In this section, we propose an iterative method for estimating the oracle probabilities, which
are required for the estimation of $q^\star$. Our proposed method brings together two key concepts:
stratification and a Bayesian generative model of the label distribution.

\subsubsection{Stratification}
Stratification is a commonly used technique in statistics that involves dividing a population into homogeneous subgroups (called strata)~\cite{cochran_sampling_1977}. Often the process of creating the strata is achieved by \emph{binning} according to a variable, or \emph{partitioning} according to a set of rules. Our use of stratification is somewhat atypical, in that we are not using it to estimate a population parameter, but rather as a \emph{parameter reduction} technique. Specifically, we aim to map the set of oracle probabilities $\{p(1|z): z\in P\}$ (of size $N=|P|$ in ER) to a smaller set of parameters of size $K$, essentially one per stratum.

\myparagraph{Parameter reduction.}
Consider a partitioning of record pair pool $P$ into $K$ disjoint strata $\{P_1, \ldots, P_K\}$, such that the pairs in a stratum share approximately the same values of $p(1|z)$.\footnote{This is the meaning of ``homogeneity'' which we adopt.} If this ideal condition is satisfied, then our work in estimating the set of probabilities $\{p(1|z):z \in P\}$ is significantly reduced, because information gained about a particular pair $z \in P_k$ is immediately transferable to the other pairs in $P_k$. As a result, we can effectively replace the set of probabilities $\{p(1|z) : z \in P_k \}$ for the record pairs in $P_k$, by a single probability $p(1|P_k)$.

\myparagraph{Relaxing the homogeneity condition.}
In reality, we don't know which record pairs in $P$ (if any) have roughly the same values of $p(1|z)$. Fortunately, it turns out that this condition does not need to be satisfied too strictly in order to be useful. Previous work~\cite{bennett_online_2010,druck_toward_2011} has demonstrated that the homogeneity condition can be satisfied in an approximate sense by using similarity scores as a proxy for true oracle probabilities. In other words, \emph{we regard a stratum to be approximately homogeneous if the pairs it contains have roughly the same similarity scores}. The more this proxy holds true, the more efficient \alg becomes in practice; however critically, our guarantees hold true regardless.

\myparagraph{Stratification method.}
In order to stratify the record pairs in $P$ according to their similarity scores, we shall use the \term{cumulative $\sqrt{F}$ (CSF) method}, originally proposed in~\cite{dalenius_minimum_1959} and previously used in the present context in~\cite{druck_toward_2011}. The CSF method has a strong theoretical grounding, in that it aims to achieve minimal intra-stratum variance in the scores. 

\begin{algorithm}[t]
	\caption{Cumulative $\sqrt{F}$ (CSF) stratification \protect\cite{dalenius_minimum_1959}}
	\label{alg:stratification-csf}
	\begin{algorithmic}[1]
		\Input \begin{tabular}[t]{cl}
			$P$ & pool of record pairs \\
			$s$ & similarity score function $: P\to\reals$  \\
			$\tilde{K}$ & desired number of strata \\
			$M$ & number of bins (for estimating score dist.)
		\end{tabular}
		\Output strata $P_1,\ldots,P_K$ (not guaranteed $K=\tilde{K}$)
		\vspace{1ex}
		\State Pool scores: $S \gets \{ s(z) | z\in P \}$
		\State \parbox[t]{\dimexpr\linewidth-\algorithmicindent}{Distribution of scores ($F$) using $M$ bins:\\ $\texttt{counts}, \texttt{score\_bins} \gets \operatorname{histogram}(S, \mathrm{bins} = M)$}
		\State Cum.\ dist.\ of $\sqrt{F}$: $\texttt{csf} \gets \big[\sum_{i = 1}^{m} \sqrt{\texttt{counts}[i]}\,\big]_{m = 1:M}$
		\State Bin width on cum. $\sqrt{F}$ scale: $w \gets \texttt{csf}[M]/\tilde{K}$
		\For {$k \in \{1, \ldots, \tilde{K} + 1\}$}
		\State Bins on cum. $\sqrt{F}$ scale: $\texttt{csf\_bins}[k] \gets (k - 1) w$
		\EndFor
		\State $K \gets 1$
		\For {$j \in \{1, \ldots, M\}$}
		\If {$K = \tilde{K}$ or $j = M$}
		\State Append $\texttt{score\_bins}[\tilde{K}]$ to \texttt{new\_bins}
		\Break
		\EndIf
		\If {$\texttt{csf}[j] \geq \texttt{csf\_bin}[K]$}
		\State Append $\texttt{score\_bins}[j]$ to \texttt{new\_bins}
		\State $K \gets K + 1$
		\EndIf
		\EndFor
		\State Allocate record pairs $P$ to strata $P_1,\ldots,P_K$ based on \texttt{new\_bins} (remove any empty strata, updating $K$)
		\State \Return $P_1,\ldots,P_K$
	\end{algorithmic}
\end{algorithm}

For completeness, we have included an implementation of the method in Algorithm~\ref{alg:stratification-csf}. It proceeds by constructing an empirical estimate of the cumulative square root of the distribution of scores (lines 2--3). Then the strata are defined as equal-width bins on the CSF scale (lines 4--7). Finally, the bins are mapped from the CSF scale to the score scale (lines 8--18), so that the scores (record pairs) may be binned in the usual way (line 19). We note that any stratification method could be used in place of the CSF method (\cf \eg the \term{equal size method} described in~\cite{druck_toward_2011}).

\myparagraph{Selecting the number of strata.}
The number of strata $K$ represents a trade-off: For large $K$, estimates of the oracle probabilities enjoy finer granularity and can better approach their true values; however large $K$ leads to more parameters and hence more labels required for convergence of estimates.

\begin{figure}
	\centering
	\includegraphics[width=\linewidth]{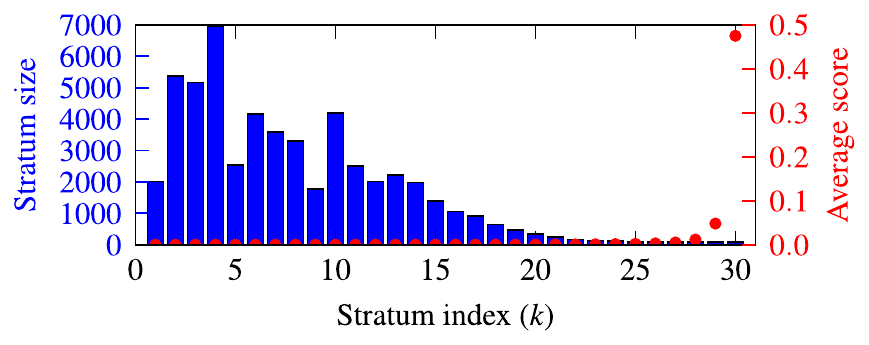}
	\caption{Size and mean score of the CSF strata for the Abt-Buy pool, using calibrated (probabilistic) scores.}
	\label{fig:csf-strata-example}
\end{figure}

In practice for ER evaluation, we find that the there is often a ``natural'' range of $K$ for the CSF method. The example in Figure~\ref{fig:csf-strata-example} shows that we typically construct very large strata with low similarity scores, and very small strata with high similarity scores: a form of heavy-tailed distribution due to the extreme class imbalance. If $K$ is set too large, then we immediately discover the strata corresponding to the higher similarity scores become too small (they may contain only 1 or 2 record pairs). We find a range of $K$ from roughly 30--60 to work well for most datasets considered in Section~\ref{sec:expt}.

\subsubsection{A Bayesian generative model}
Having partitioned the record pairs in $P$ into $K$ strata $\{P_1, \ldots, P_K\}$, our goal is to estimate $p(1|P_k)$ (for all $k$) using the collected \oracle labels. For notational convenience, we denote the true value of $p(1|P_k)$ by $\pi_k$ and a corresponding estimate by $\hat{\pi}_k$. We shall adopt a generative model for observed labels which regards $\pi_k$ as a latent variable.

\myparagraph{Model of a stratum.} Consider a label $\ell$ received from the oracle for a record pair in $P_k$. We assume that the label is generated from a Bernoulli distribution with probability $\pi_k$ of being a match (binary label `1'), \ie
\begin{equation}
	\ell \sim \mathrm{Bernoulli}(\pi_k)\enspace.
\end{equation}
Since the Bernoulli distribution is conjugate to the beta distribution, we adopt a beta prior for $\pi_k$:
\begin{equation}
\pi_{k} \sim \mathrm{Beta}(\gamma_{0,k}^{(0)}, \gamma_{1,k}^{(0)})\enspace,
\end{equation}
where $\gamma_{0,k}^{(0)}$ and $\gamma_{1,k}^{(0)}$ are the prior hyperparameters. We describe how to choose the prior hyperparameters in Sections~\ref{sec:alg-initialisation} and \ref{sec:alg-finish}.

\myparagraph{Joint model of strata.} To model each stratum independently but not identically---we do not transfer information across strata but grant each a prior---we factor the joint prior distribution as a product of the marginal $K$ priors. We collect the $\pi_k$'s into a vector $\vec{\pi} = [\pi_1, \pi_2, \ldots, \pi_K]$ and the prior hyperparameters into a $2\times K$ matrix:
\begin{equation}
\vec{\Gamma}^{(0)} = \begin{bmatrix}
\gamma_{0,1}^{(0)} & \gamma_{0,2}^{(0)} & \cdots & \gamma_{0,K}^{(0)} \\
\gamma_{1,1}^{(0)} & \gamma_{1,2}^{(0)} & \cdots & \gamma_{1,K}^{(0)}
\end{bmatrix}\enspace.
\end{equation}
The posterior distribution of $\vec{\pi}$, given the labels received from the oracle up to iteration $t$, is a product of the $K$ corresponding independent beta posterior distributions. Continuing with the previous notation, we store the posterior hyperparameters at iteration $t$ in a matrix $\vec{\Gamma}^{{(t)}}$.

\myparagraph{Iterative posterior updates.} To obtain a new estimate of $\vec{\pi}$ per iteration, we iteratively update the posterior hyperparameters upon arrival of \oracle label observations $\ell_t$. Suppose label $\ell_t$ is observed as a result of querying with a record pair from stratum $P_{k^\star}$. Then the update involves:
\begin{equation}
\begin{split}
\text{copy old values}&: \quad \vec{\Gamma}^{(t)} \gets \vec{\Gamma}^{(t-1)}\\
\text{if } \ell_t = 1&: \quad \gamma_{0,k^\star}^{(t)} \mathrel{{+}{=}} 1 \\
\text{if } \ell_t = 0&: \quad \gamma_{1,k^\star}^{(t)} \mathrel{{+}{=}} 1
\end{split}
\label{eqn:hyperparameter-update-rule}
\end{equation}

A point estimate of $\vec{\pi}$ can be obtained at iteration $t$ via the posterior mean
\begin{equation}
\hat{\vec{\pi}}^{(t)} = \operatorname{E}[\vec{\pi}|\ell_1,\ldots,\ell_t] = \frac{\vec{\Gamma}_{0,:}^{(t)}}{\vec{\Gamma}_{0,:}^{(t)} + \vec{\Gamma}_{1,:}^{(t)}}\enspace.
\end{equation}
Here the notation $\vec{\Gamma}_{i,:}^{(t)}$ represents the $i$-th row of matrix $\vec{\Gamma}^{(t)}$, and the division is carried out element-wise.

\begin{remark}
As a practical modification to speed up convergence of $\hat{\vec{\pi}}$, we can decrease our
reliance on the prior as labels are received. For each column $\vec{\Gamma}_{:,k}^{(0)}$ we can
retroactively multiply by a factor $1/n_k$ where $n_k$ is the number of labels sampled from $P_k$
thus far. Anecdotally we also observe that this improves robustness to misspecified priors.
\end{remark}

\subsubsection{Stratified instrumental distribution}
Since the estimation method for the oracle probabilities produces estimates over the strata, rather
than for individual pairs in the pool, it is appropriate to estimate the instrumental distribution
in the same way. Akin to the mapping from $p(1|z)$ to $\vec{\pi}$, we therefore propose to map $q(z)$ to a vector $\vec{v} = [v_1, \ldots, v_K]$ based on our Bayesian stratified model estimates instead of (unknowable) population parameters. Adapting Eqn.~\eqref{eqn:optimal-inst-dist}, the stratified asymptotically optimal instrumental distribution $\vec{v}^\star$ is defined at iteration $t$ as
\begin{equation*}
\begin{split}
v_{k}^{\star (t)} &\propto \omega_k \left[ (1- \alpha) (1 -  \lambda_k) \hat{F}_{\alpha}^{(t-1)} \sqrt{\hat{\pi}_k^{(t-1)}} \right. \\
& \left. {} + \lambda_k \sqrt{(\alpha \hat{F}_\alpha^{(t-1)})^2 (1 - \hat{\pi}_k^{(t-1)}) + (1 - \hat{F}_\alpha^{(t-1)})^2 \hat{\pi}_k^{(t-1)}} \right],
\end{split}
\end{equation*}
where $\omega_k = |P_k|/N$ is the \term{weight} associated with $P_k$ and $\lambda_k = \frac{1}{\left| P_k \right|}\sum_{i \in P_k} \hat{\ell}_{i}$ is the mean prediction in $P_k$. It follows that the $\varepsilon$-greedy distribution at iteration $t$ is given by
\begin{equation}
v_{k}^{(t)} = \varepsilon \cdot \omega_k + (1 - \varepsilon) \cdot v_{k}^{\star (t)}\enspace.
\label{eqn:alg-epsilon-greedy-inst-dist}
\end{equation}

Having adopted a stratified representation for the instrumental distribution, sampling a record pair is now a two-step process. First a stratum index is drawn from $\{1,\ldots, K\}$ according to $\vec{v}$. Then a record pair is drawn uniformly at random from the resulting stratum.

\subsection{Initialisation}
\label{sec:alg-initialisation}
\alg requires a set of prior hyperparameters $\vec{\Gamma}^{(0)}$ and a guess for the F-measure $\hat{F}_\alpha^{(0)}$ for initialisation purposes. We elect to set these quantities based on the information contained within the similarity scores. Our approach depends centrally on a guess for the oracle probabilities $\hat{\vec{\pi}}^{(0)}$, in that once $\hat{\vec{\pi}}^{(0)}$ is available, the values of $\vec{\Gamma}^{(0)}$ and $\hat{F}_\alpha^{(0)}$ immediately follow. The details of the initialisation are contained in Algorithm~\ref{alg:initialisation}, with further explanation given below.

\myparagraph{Oracle probabilities (lines 2--5).} 
A reasonable guess for $\vec{\pi}$ can be obtained by taking the mean of the similarity scores in each stratum. If the scores are not probabilities, they should be mapped to the $[0,1]$ interval. This can be achieved by applying the logistic function.

\myparagraph{F-measure (lines 6 \& 8).}
The calculation of $\hat{F}_\alpha^{(0)}$ depends on the guess for $\vec{\pi}$ described above and the mean prediction per stratum $\vec{\lambda}$. Breaking down the calculation term-by-term, one begins by estimating the probability of finding a true positive in $P_k$ as $\hat{\pi}_k^{(0)} \lambda_k$, so that the total number of true positives may be approximated by $\sum_{k = 1}^{K} |P_k| \hat{\pi}_k^{(0)} \lambda_k$. Similarly, the total number of actual positives (TP + FN) may be approximated by $\sum_{k = 1}^{K} |P_k| \hat{\pi}_k^{(0)}$. The total number of predicted positives (TP + FP) is known exactly and can be written in terms of $\vec{\lambda}$ as $\sum_{k = 1}^{K} |P_k| \lambda_k$. Using these estimates in Eqn.~\eqref{eqn:F-measure} yields the guess for $\hat{F}_\alpha^{(0)}$ in line 8.

\myparagraph{Prior hyperparameters.}
We also set $\vec{\Gamma}^{(0)}$ based on $\hat{\vec{\pi}}^{(0)}$ 
\begin{equation*}
\vec{\Gamma}^{(0)} = \eta \begin{bmatrix}
\hat{\vec{\pi}}^{(0)} \\
1 - \hat{\vec{\pi}}^{(0)}
\end{bmatrix}\enspace.
\end{equation*}
Here $\eta > 0$ is an adjustable parameter that controls the strength of the prior. For ease of presentation, this step is included in Algorithm~\ref{alg:AIS} (line 1) rather than Algorithm~\ref{alg:initialisation}.

\begin{algorithm}[t]
	\caption{Initialisation of Bayesian model}
	\label{alg:initialisation}
	\begin{algorithmic}[1]
		\Input \begin{tabular}[t]{cl}
			$0 \leq \alpha \leq 1$ & F-measure weight \\
			$P$ & pool of record pairs \\
			$\hat{\cR}$ & predicted ER \\
			$s$ & similarity score function $: P\to\reals$ \\
			$\tau$ & \reals-valued score threshold (optional) \\
		$\{P_k\}_{k=1}^K$& stratum allocations
		\end{tabular}
		\Output \begin{tabular}[t]{cl}
			$\hat{F}_{\alpha}^{(0)}$ & initial F-measure \\
			$\hat{\vec{\pi}}^{(0)}$ & prior hyperparameters
		\end{tabular}
		\vspace{1ex}
		\For {$k \in \{1,\ldots, K\}$}
		\State Mean score per stratum: $\hat{\pi}_{k}^{(0)} \gets \frac{1}{\left| P_k \right|}\sum_{z \in P_k} s(z)$
		\If {scores are not probabilities in $[0,1]$}
		\State Transform: $\hat{\pi}_{k}^{(0)} \gets \mathrm{logit}(\hat{\pi}_{k}^{(0)} - \tau)$
		\EndIf
		\State Mean pred.\ per stratum: $\mathbf{\lambda}_k \gets \frac{1}{\left| P_k \right|}\sum_{z \in P_k} \hat{\ell}_{z}$
		\EndFor
		\State $\hat{F}_{\alpha}^{(0)} \gets \frac{\sum_{k = 1}^{K} \left| P_k \right| \pi_{0,k}\lambda_k}{\alpha \sum_{k = 1}^{K} \left| P_k \right| \lambda_k + (1 - \alpha) \sum_{k = 1}^{K} \left| P_k \right| \pi_{0,k} }$
		\State \Return $\hat{F}_{\alpha}^{(0)}$, $\hat{\vec{\pi}}^{(0)}$
	\end{algorithmic}
\end{algorithm}

\subsection{Bringing everything together}
\label{sec:alg-finish}
Having introduced all of the components of \alg, we are now ready to explain how they fit together. Recall that the evaluation process begins with three main inputs: the pool of record pairs $P$, similarity scores $s(\cdot)$, and predicted ER $\hat{\cR}$. A summary of the main steps involved is as follows:

\begin{enumerate}[(i)]
	\item Generate a set of strata $P_1,\ldots,P_K$ partitioning $P$ using the CSF method (Algorithm~\ref{alg:stratification-csf}).
	\item Generate initial estimates using the strata, $\hat{\cR}$ and $s(\cdot)$ (Algorithm~\ref{alg:initialisation}).
	\item Conduct AIS to estimate $F_\alpha$ (Algorithm~\ref{alg:AIS}).
\end{enumerate}

\myparagraph{Summary of Algorithm~\ref{alg:AIS}.} At each iteration $t$: sample a stratum according to $\vec{v}^{(t)}$, then a record pair within that stratum uniformly at random. Query \oracle for a label of the record pair. Use the observed label (and the predicted label) to update the oracle probabilities  (using Eqn.~\ref{eqn:hyperparameter-update-rule}) and the F-measure estimate (using Eqn.~\ref{eqn:AIS-F-estimator}). Stop after $T$ iterations and return the final estimate $\hat{F}_{\alpha}^{(T)}$.

\begin{algorithm}[t]
	\caption{\alg for estimation of the F-measure}
	\label{alg:AIS}
	\begin{algorithmic}[1]
		\Input \begin{tabular}[t]{cl}
			$T > 0$ & number of iterations \\
			$0 \leq \alpha \leq 1$ & F-measure weight \\
			$0 < \varepsilon \leq 1$ & greediness parameter \\
			$\eta > 0$ & prior strength parameter \\
			$\hat{F}_{\alpha}^{(0)}$ & initial guess for F-measure \\
			$\hat{\vec{\pi}}^{(0)}$ & initial guess for pos. probabilities \\
			$\hat{\cR}$ & predicted ER \\
		$\{P_{k}\}_{k=1}^K$ & stratum allocations \\
			\texttt{Oracle} & randomised (noisy) true labels
		\end{tabular}
		\Output \begin{tabular}[t]{cl}
			$\hat{F}_{\alpha}^{(T)}$ & F-measure estimate
		\end{tabular}
		\vspace{1ex}
		\State $\vec{\Gamma} \gets \eta \begin{bmatrix}
		\hat{\vec{\pi}}^{(0)} \\
		1 - \hat{\vec{\pi}}^{(0)}
		\end{bmatrix}$ \Comment{initialise Bayesian model}
		\For{$t \in \{1, \ldots, T\}$}
		\State Calculate $\vec{v}^{(t)}$ using Eqn.~\eqref{eqn:alg-epsilon-greedy-inst-dist}
		\State Draw $k^\star$ from $\{1,\ldots,K\}$ according to $\vec{v}^{(t)}$
		\State Draw $z^\star$ from $P_{k^\star}$ uniformly
		\State $w_t \gets \omega_k/ v_k^{(t)}$ \Comment{importance weight}
		\State $\ell_t \gets \oracle(z^\star)$ \Comment{query label from oracle}
		\State $\hat{\ell}_t \gets \hat{\ell}(z^\star)$ \Comment{record prediction}
		\State $\vec{\Gamma}_{:,k^\star} \gets \vec{\Gamma}_{:,k^\star} + \begin{bmatrix}
		\ell_t\\
		1 - \ell_t
		\end{bmatrix}$ \Comment{update posterior}
		\State $\hat{\vec{\pi}}^{(t)} \gets \vec{\Gamma}_{0,:} \mathrel{{.}{/}} (\vec{\Gamma}_{0,:} + \vec{\Gamma}_{1,:})$ \Comment{update $\vec{\pi}$ estimate}
		\State $\hat{F}_{\alpha}^{(t)} \gets \frac{\sum_{\tau = 0}^{t} w_\tau \ell_\tau \hat{\ell}_\tau}{\alpha \sum_{\tau = 0}^{t} w_\tau \hat{\ell}_\tau + (1 - \alpha)\sum_{\tau = 0}^{t} w_\tau \ell_\tau}$
		\EndFor
		\State \Return $\hat{F}_{\alpha}^{(T)}$
	\end{algorithmic}
\end{algorithm}

\section{Consistency of \alg}
\label{sec:consistency-analysis}
A fundamental requirement of any well-behaved estimation procedure is \emph{consistency}, that is,
given enough samples we want the estimate to be close to the true value with high probability. 
Nominated as one of our objectives in designing the \alg algorithm in
Section~\ref{sec:problem-formulation}, we now prove that \alg is statistically consistent.

Before we begin, we acknowledge previous theoretical work on the consistency of other AIS
algorithms, notably Population Monte Carlo (PMC)~\cite{cappe_population_2004,douc_convergence_2007,cappe_adaptive_2008} and Adaptive Multiple Importance Sampling (AMIS)~\cite{cornuet_adaptive_2012,marin_consistency_2012}. Unfortunately, we cannot directly apply these results here owing to the following differences in our setup: 
\begin{enumerate}[(i)]
	\item we do not discard and re-draw the entire sample at each iteration since it would waste our label budget;
	\item we permit the instrumental distribution to be updated based on samples from \emph{all} previous iterations (unlike \cite{douc_convergence_2007,cappe_adaptive_2008} which are restricted to the previous iterate);
	\item we examine consistency as $T \to \infty$ (others assume that the sample size increases at each iteration and examine consistency in this limit).
\end{enumerate}

Due to the dependent nature of the sample and the non-linear form of the F-measure, the proof is relatively involved and requires some build-up. In Section~\ref{sec:consistency-sample-averages}, we first consider simple AIS estimators based on sample averages, and show that strong consistency follows so long as some reasonable conditions are met. Then in Section~\ref{sec:consistency-F-measure} we extend these results to the non-linear F-measure estimator. Until this point, we assume a general instrumental distribution and updating mechanism, before finally specialising to the \alg method in Section~\ref{sec:consistency-alg}.

\subsection{Simple AIS estimators}
\label{sec:consistency-sample-averages}
Consider a random variable $X$ with probability density $p(x)$ and consider the estimation of parameter $\theta = \operatorname{E}[f(X)]$ using AIS. This involves constructing sample $\{x_1, x_2, \ldots, x_T\}$ by drawing each item sequentially from a separate instrumental distribution. Specifically, we assume that the $t$-th sample $x_t$ is drawn from an instrumental distribution with density $q_t(x_t | \vec{x}_{1:t-1})$ which depends on the $t-1$ previously sampled items $\vec{x}_{1:t-1} = \{x_1, \ldots, x_{t-1} \}$.\footnote{Beginning with an initial sampling distribution $q_1(\vec{x}_1)$.} The AIS estimator of $\theta$ is then defined as:
\begin{equation}
	\hat{\theta}^{\mathrm{AIS}} = \frac{1}{T} \sum_{t = 1}^{T} w_t  f(x_t),
	\label{eqn:AIS-estimator}
\end{equation}
which may be interpreted as an importance-weighted sample average. Here the importance weights are given by $w_t = p(x_t)/q_t(x_t)$ (we omit the conditioning on $\vec{x}_{1:t-1}$ for notational simplicity).

In order to prove that $\hat{\theta}^\mathrm{AIS}$ is consistent for $\theta$, we rely on the following lemma, which generalises the law of large numbers (LLN) to history-dependent random sequences.
\begin{lemma}
	\label{lem:LLN}
	Let $\{U_t\}_{t = 1}^{\infty}$ be a sequence of random variables and let $\vec{U}_{1:T} = \{U_{1}, U_{2}, \ldots, U_{T}\}$ denote the sequence up to index $t = T$. Suppose that the following conditions hold:
	\begin{enumerate}[(i)]
		\item $\operatorname{E}[U_1] = \theta$; \label{lem:LLN-c1}
		\item $\operatorname{E}[U_t | \vec{U}_{1:t-1}] = \theta$ for all $t > 1$; and \label{lem:LLN-c2}
		\item $\operatorname{E}[U_t^2] \leq C < \infty$ for all $t \geq 1$.\label{lem:LLN-c3}
	\end{enumerate}
	Then $\frac{1}{T} \sum_{t = 1}^{T} U_t \to \theta$ almost surely.
\end{lemma}
The proof of this lemma is given in the 
\ifdefined\ARXIV
appendix, 
\else
full report~\cite{techreport}, 
\fi
and relies on a more general theorem due to Petrov~\cite{petrov_strong_2014}.

By observing that the summands in Eqn.~\eqref{eqn:AIS-estimator} obey conditions~(i) and (ii) of Lemma~\ref{lem:LLN}, we can establish the following theorem on the strong consistency of $\hat{\theta}^{\mathrm{AIS}}$.
\begin{theorem}
	\label{thm:consistency-linear-AIS}
	The estimator in Eqn.~\eqref{eqn:AIS-estimator} is \emph{strongly consistent}, that is, $\hat{\theta}^\mathrm{AIS} \to \theta$ almost surely, provided the following conditions are met for all $t \geq 1$:
	\begin{enumerate}[(i)]
		\item $q_t(x) > 0$ whenever $f(x) p(x) \neq 0$, and
		\item $\underset{\substack{X_t \sim p \\ \vec{X}_{1:t-1} \sim g}}{\operatorname{E}} \left[\frac{p(X_t)}{q_t(X_t)} f(X_t)^2 \right] \leq C < \infty$.
	\end{enumerate}
\end{theorem}

\begin{proof}
	Let $U_t = \frac{p(X_t)}{q_t(X_t|\vec{X}_{1:t-1})} f(X_t)$ and $\theta = \operatorname{E}[f(X)]$. The almost sure convergence follows by checking the conditions of Lemma~\ref{lem:LLN}. For condition (ii) of the lemma, we find
	\begin{align*}
	\operatorname{E}\!\left[U_t\middle|\vec{U}_{1:t-1}\right] & = \operatorname{E}\!\left[\frac{p(X_t)}{q_t(X_t|\vec{X}_{1:t-1})} f(X_t) \middle | \vec{X}_{1:t-1} \right]\\
	& = \int_{\mathcal{X}} \frac{p(x_t)}{q_t(x_t | \vec{x}_{1:t-1})} f(x_t) q_t(x_t|\vec{x}_{1:t-1})\, d x_t \\
	& = \int_{\mathcal{X}} f(x) p(x) \, d x \tag{by condition \ref{lem:LLN-c1}}\\
	& = \operatorname{E}\!\left[f(X)\right] = \theta.
	\end{align*}
	Condition (i) of the lemma follows by a similar argument.
	
	Finally we check condition (iii): that the second moment is bounded. Denoting the joint density of $\vec{X}_{1:t-1}$ by $g$ and considering $t>1$, we have
	\begin{align*}
		& \operatorname{E}\!\left[U_t^2\right] \\  = &\operatorname{E}\!\left[\operatorname{E} \!\left[U_t^2 | \vec{U}_{1:t-1}\right] \right]\\	
		=&\operatorname{E} \! \left[\operatorname{E} \! \left[\left(\frac{p(X_t)}{q_t(X_t | \vec{X}_{1:t-1})} f(X_t)\right)^2 \middle| \vec{X}_{1:t-1} \right]\right] \\
		=& \iint_{\mathcal{X}} \left(\frac{p(x_t)f(x_t)}{q_t(x_t | \vec{x}_{1:t-1})}  \right)^2 q_t(x_t|\vec{x}_{1:t-1}) d x_t 
	g(\vec{x}_{1:t-1}) d\vec{x}_{1:t-1} \\
		=& \iint_{\mathcal{X}} \frac{p(x_t) f(x_t)^2}{q_t(x_t | \vec{x}_{1:t-1})}  p(x_t)\, d x_t \ g(\vec{x}_{1:t-1}) d\vec{x}_{1:t-1} \tag{by (i)}\\
		=& \underset{\substack{X_t \sim p \\ \vec{X}_{1:t-1} \sim g}}{\operatorname{E}} \left[\frac{p(X_t)}{q_t(X_t|\vec{X}_{1:t-1})} f(X_t)^2 \right]
	\end{align*}
	which is bounded above by assumption. This also holds for $t=1$ (by the above argument without the sampling history). Thus all of the conditions of Lemma~\ref{lem:LLN} are satisfied, and the proof is complete.
\end{proof}

\subsection{The AIS F-measure estimator}
\label{sec:consistency-F-measure}
The AIS estimator for the F-measure, $\hat{F}_\alpha^{\mathrm{AIS}}$, is less straightforward to analyse because it cannot be expressed as a sample average like the estimators studied in Section~\ref{sec:consistency-sample-averages}. Instead, we regard $\hat{F}_\alpha^{\mathrm{AIS}}$ as a \emph{ratio} of sample averages:
\begin{equation*}
	\hat{F}_\alpha^{\mathrm{AIS}} = \frac{\frac{1}{T}\sum_{t = 1}^{T} w_t f_\mathrm{num}(x_t)}{\frac{1}{T}\sum_{t = 1}^{T} w_t f_\mathrm{den}(x_t)},
\end{equation*}
 (\cf Eqn.~\ref{eqn:AIS-F-estimator}) where $x_t = (z_t, \ell_t)$ denotes a record pair and its observed label, and the functions are 
\begin{equation}
\begin{split}
	f_\mathrm{num}(x_t) &= \ell_t \hat{\ell}_t\enspace ; \quad \text{and} \\
	f_\mathrm{den}(x_t) &= \alpha \hat{\ell}_t + (1 - \alpha) \ell_t\enspace.
\end{split}
\label{eqn:definition-fnum-fden}
\end{equation}

We leverage Theorem~\ref{thm:consistency-linear-AIS} to show that the numerator and denominator
both converge to their respective true values, which is sufficient to establish convergence of
$\hat{F}_\alpha^{\mathrm{AIS}}$.

%Before we can carry out the convergence analysis, we need to address the issue of what is in fact the ``true'' value of the F-measure. Recall that the F-measure as defined in Eqn.~\eqref{eqn:F-measure} is evaluated with respect to a pool of record pairs (i.e. test set). This means that it's value will vary depending on the particular pool that is used for the evaluation (assuming the pool is drawn randomly from a larger population). We can get around this random variation by defining a generalised F-measure, $F_\alpha^{\mathrm{true}}$, which takes into account how the ER system performs across the entire space of record pairs and labels (following~\cite{sawade_active_2010}):
%\begin{equation}
% F_\alpha^{\mathrm{true}} = \frac{\operatorname{E}[f_\mathrm{num}(X)]}{\operatorname{E}[f_\mathrm{den}(X)]}.
%\label{eqn:F-true}
%\end{equation}
%Here $X = (Z, L)$ denotes a random record pair and label, and the expectation is taken with respect to the joint distribution of $X$. One can prove that the ordinary F-measure, $F_\alpha$, is consistent for $F_\alpha^{\mathrm{true}}$ provided sampling is conducted over the entire space (see \cite{sawade_active_2010} for a proof). The following theorem asserts that the AIS estimator for the F-measure, is also consistent for $F_\alpha^{\mathrm{true}}$.
\begin{theorem}
	\label{thm:consistency-F-alg}
	Let $X = (Z, L)$ denote a random record pair $Z$ and its corresponding label $L$, and let the density of $X$ be $p(x) = p(\ell | z) p(z) $. Suppose AIS is carried out to estimate the F-measure and assume that the conditions of Theorem~\ref{thm:consistency-linear-AIS} are satisfied by $p(x)$ and $q_t(x)$ for both functions defined in Eqn.~\eqref{eqn:definition-fnum-fden}. Assume furthermore that the instrumental density can be factorised as $q_t(x_t| \vec{x}_{1:t-1}) = p(\ell_t | \vec{z}_t) q_t(z_t|\vec{x}_{1:t-1})$ for all $t \geq 1$. Then $\hat{F}_\alpha^\mathrm{AIS}$ is \emph{weakly consistent} for $F_\alpha$.
\end{theorem}
\begin{proof}
	Observe that for the numerator of $\hat{F}_\alpha^{\mathrm{AIS}}$,
	\begin{equation*}
	\frac{1}{T} \sum_{t = 1}^{T} \frac{p(Z_t)}{q_t(Z_t)} f_\mathrm{num}(X_t) = \frac{1}{T} \sum_{t = 1}^{T} \frac{p(X_t)}{q_t(X_t)} f_\mathrm{num}(X_t)
	\end{equation*}
	using the factorised form of $q_t(x)$. This converges in probability to $\operatorname{E}[f_\mathrm{num}(X)]$ by Theorem~\ref{thm:consistency-linear-AIS}. The same is true for the denominator (replace $f_\mathrm{num}$ by $f_\mathrm{den}$). Invoking Slutsky's theorem, we have
	\begin{equation*}
	\hat{F}_\alpha^{\mathrm{AIS}} = \frac{\frac{1}{T}\sum_{t = 1}^{T} \frac{p(Z_t)}{q_t(Z_t)} f_\mathrm{num}(X_t)}{\frac{1}{T}\sum_{t = 1}^{T} \frac{p(Z_t)}{q_t(Z_t)} f_\mathrm{den}(X_t)} \overset{P}{\longrightarrow} \frac{\operatorname{E}[f_\mathrm{num}(X)]}{\operatorname{E}[f_\mathrm{den}(X)]}
	\end{equation*}
	It is straightforward to show that the expression on the right-hand side reduces to $F_\alpha$ by evaluating the expectations with respect to $p$ for finite pool $P$. For the more general case,
	it can be shown that the F-measure statistics $F_{\alpha,T}$ converge to the right-hand side
	population-based F-measure~\cite{sawade_active_2010}.
\end{proof}

%\begin{remark}
%	The generalised F-measure $F_\alpha^{\mathrm{true}}$ coincides with the ordinary F-measure $F_\alpha$ if the pool $P$ contains the \emph{entire} population of record pairs, and the labels provided by the oracle are \emph{certain} (contain no noise). This is easy to show by plugging the following discrete distribution
%	\begin{equation*}
%		p(x) = \frac{1}{N} \sum_{i = 1}^{N} \mathbb{I}[\ell = \ell_i] \mathbb{I}[z = z_i]
%	\end{equation*}
%	into Eqn.~\eqref{eqn:F-true}. \hl{In subsequent experiments, just want to evaluate how AIS performs on a single pool $P$ (don't keep sampling a new pool). In this case, assume $P$ is the entire population so we can talk about convergence to $F_\alpha$}
%\end{remark}

\subsection{Application to \alg}
\label{sec:consistency-alg}
Theorem~\ref{thm:consistency-F-alg} tells us about the convergence of $\hat{F}_\alpha^{\mathrm{AIS}}$ for \emph{any choice of instrumental distribution and update mechanism meeting the conditions}. Our final remaining task is to show that these conditions are met by Algorithm~\ref{alg:AIS}.
\begin{theorem}
	Algorithm~\ref{alg:AIS} (\alg) produces a consistent estimate of $F_\alpha$, that is $\hat{F}_\alpha^{(T)} \overset{P}{\to} F_\alpha$.
	\label{thm:consistency-alg}
\end{theorem}

The proof is straightforward, while lengthy, and so is relegated to 
\ifdefined\ARXIV
the appendix. 
\else
the full report~\cite{techreport}. 
\fi
It proceeds by checking that the conditions of Theorem~\ref{thm:consistency-F-alg} are satisfied
by the \alg instrumental distribution.

\begin{remark}
It is now apparent why we adopt the $\varepsilon$-greedy instrumental distribution: while
$q_t^\star(z)$ can go to zero when $p(z) \neq 0$, violating condition~(i) of
Theorem~\ref{thm:consistency-linear-AIS}, $\varepsilon$-greedy cannot. For example, if
$\hat{\pi}_k = 0$ and $\lambda_k = 0$ then $q_t^\star(z) = 0$ for all $z \in P_k$, whilst $p(z) = 1/N \neq 0$. The $\varepsilon$-greedy instrumental distribution does not vanish since $q_t(z) = \varepsilon/N > 0$.
\end{remark}

\begin{table}
	\centering
	\caption{Datasets in decreasing order of class imbalance. The size of the dataset is the number of record pairs it contains and the imbalance ratio is the ratio of non-matches to matches. The $\star$ indicates that the dataset is not from the ER domain.}
	\begin{tabular}{ >{\ttfamily}l r c c}
		\hline
		\multicolumn{1}{l}{\multirow{2}{*}{Dataset Name}} & \multicolumn{1}{c}{\multirow{2}{*}{Size}} & Imb. & No. \\
		& & Ratio & Matches \\
		\hline
		Amazon-GoogleProducts & 4,397,038 & 3381 & 1300 \\
		restaurant            &   745,632 & 3328 & 224 \\
		DBLP-ACM              & 5,998,880 & 2697 & 2224 \\
		Abt-Buy               & 1,180,452 & 1075 & 1097 \\
		cora                  & 1,675,730 & 47.76 & 34,368 \\
		$\star$ tweets100k    &   100,000 & 1 & 50,000 \\
		\hline
	\end{tabular}
	\label{tbl:datasets}
\end{table}

\section{Experiments}\label{sec:expt}
In this section, we examine whether \alg addresses our main objective of reducing labelling requirements for evaluating ER. We run comprehensive experiments comparing \alg with established methods, which conclusively establish that \alg is generally superior, requiring significantly fewer labels to achieve a given precision of estimate. \\

\subsection{Experimental setup}
\subsubsection{Datasets}
We use five publicly available ER datasets as listed in Table~\ref{tbl:datasets}. All datasets come with true resolution $\cR$. \texttt{Abt-Buy}~\cite{kopcke_evaluation_2010} and \texttt{Amazon-GoogleProducts}~\cite{kopcke_evaluation_2010} are from the e-commerce domain; \texttt{cora}~\cite{riddle_datasets} and \texttt{DBLP-ACM}~\cite{kopcke_evaluation_2010} relate to computer science citations; and \texttt{restaurant} contains listings from two restaurant guidebooks~\cite{riddle_datasets}. We note that \texttt{cora} is unique among these datasets, in that it does not arise from two separate DBs. Technically, it is an example of \term{de-duplication}, which may be cast as ER on the DB matched with itself.
% Abt-Buy, DBLP-ACM, Amazon-GoogleProducts available from http://dbs.uni-leipzig.de/en/research/projects/object_matching/fever/benchmark_datasets_for_entity_resolution
% restaurant, cora available from http://www.cs.utexas.edu/users/ml/riddle/data.html
% \texttt{restaurant} is a database of 864 restaurant names and addresses containing 112 duplicates assembled by Sheila Tejada from Fodor's and Zagat's	guidebooks. \texttt{cora} is a collection of 1295 distinct references to 122 Computer Science	research papers from the Cora Computer Science research paper search engine provided by William Cohen.

In addition to these five datasets, we have also included \texttt{tweets100k}~\cite{mozafari_scaling_2014} from outside the ER domain. It is included to test whether the sampling methods are competitive in the \emph{absence} of class imbalance.
% Available from http://web.eecs.umich.edu/~mozafari/datasets/crowdsourcing/

\myparagraph{Pooling.}
Although evaluation is ideally conducted with respect to the entire pool, $P = \cZ$, a key baseline sampling method (IS, introduced in Section~\ref{sec:expt-baseline-methods}) is prohibitively slow for such large pools (\cf Section~\ref{sec:expt-timing}) since its instrumental distribution is defined on each record pair. \alg does not suffer from this drawback and runs efficiently on entire pools. However to complete a fair comparison, we opt to conduct the evaluation with respect to smaller pools drawn randomly from $\cZ$, which are listed in Table~\ref{tbl:pools}. This does not affect the validity of the theory/algorithm; indeed relative to (significant) randomised pools, $F_\alpha$ is with high probability exceedingly close to that defined relative to \cZ.

\myparagraph{Oracle.}
We implement an oracle based on the ground truth resolution $\cR$ provided per dataset. Since only one label is provided per record pair, we are in the regime of a deterministic \oracle \ie with probabilities $p(1|z)\in\{0,1\}$.

\subsubsection{ER pipeline}
We build a simple ER pipeline with the following features:

\myparagraph{Pre-processing.} Strings are normalised by removing symbols, accents \& capitalisation. Numeric fields are converted to floats and missing values are imputed using the mean.

\myparagraph{Similarity features.} For each pair of fields (\eg the `Name' fields of $\cD_1$ and $\cD_2$) we calculate a scalar feature based on some measure of their similarity. For short textual fields we the Jaccard distance based on trigrams and for long textual fields we use cosine similarity with a tf-idf vector representation. For numeric fields we use the normalised absolute difference.

\myparagraph{Record pair classifier.} At the core of the ER pipeline is a binary classifier, which operates on the space of similarity features. We generally use a linear SVM (L-SVM), trained on a random subset of the entire dataset (including ground truth labels). Since we would like to test the evaluation in a range of circumstances, we don't always aim for the best classifier---we instead aim for a range of classifiers with excellent performance through to poor.

\subsection{Baseline methods}
\label{sec:expt-baseline-methods}
We compare \alg with three baseline methods.

\myparagraph{Passive.} This simple method samples record pairs uniformly at random from the pool with replacement. At each iteration, the F-measure is estimated using Eqn.~\eqref{eqn:F-measure-finite}, based only on the record pairs\slash labels sampled so far.

\myparagraph{Stratified.} This method has been used previously in~\cite{druck_toward_2011} for estimating balanced F-measures. It involves partitioning the pool of record pairs into strata (we set $K=30$) using Algorithm~\ref{alg:stratification-csf}. Record pairs are then sampled by drawing a stratum according to the stratum weights ($\omega_k = |P_k|/N$), then sampling within the stratum uniformly. The F-measure is estimated using a stratified version of Eqn.~\eqref{eqn:F-measure} (see~\cite{druck_toward_2011}).
% When estimating accuracy, they estimate the variance of $P_k$ using $\hat{\sigma}_k^2 \approx \pi_k(1-\pi_k)$ and $\hat{\pi}_k = \frac{\sum_{i = 1}^{N_k} \ell_{i;k} + \sum_{i = 1}^{M_k}\tilde{\ell}_{i;k}}{N_k+M_k}$ where $\alpha = 1/\sqrt{N_k}$ (or 2 if $N_k = 0$). The $\tilde{\ell}_{i;k}$ are pseudo-observations drawn from $\mathrm{Bernoulli}(s)$ where $s$ is a randomly sampled score from the unlabelled items.

\myparagraph{IS.} Non-adaptive importance sampling has been used for evaluating F-measures in~\cite{sawade_active_2010}: record pairs are sampled according to a \emph{static} instrumental distribution which aims to approximate Eqn.~\eqref{eqn:optimal-inst-dist}. IS may be far from optimal depending on score reliability, since the approximation replaces $p(1|z)$ with the similarity scores (mapped to the unit interval). The estimate of the F-measure is obtained at each iteration using a static version of Eqn.~\eqref{eqn:AIS-F-estimator}.

\begin{table*}
	\centering
	\caption{Pools sampled from the datasets in Table~\ref{tbl:datasets}, along with the true performance measures.}
	\begin{tabular}{>{\ttfamily}lrcccccc}
		\hline
		\multicolumn{1}{l}{Associated Dataset} & \multicolumn{1}{c}{Size} & Imb.\ ratio & No.\ matches & Classifier & Precision & Recall & $F_{1/2}$ \\
		\hline
		Amazon-GoogleProducts & 676,267 & 3381 & 200 & L-SVM & 0.597 & 0.185 & 0.282 \\
		restaurant            & 149,747 & 3328 & 45 & L-SVM & 0.909 & 0.888 & 0.899 \\
		DBLP-ACM              &  53,946 & 2697 & 20 & L-SVM & 1.0 & 0.9 & 0.947 \\
		Abt-Buy               &  53,753 & 1075 & 50 & L-SVM & 0.916 & 0.44 & 0.595 \\
		%Abt-Buy               &  53,753 & 1075 & 50 & R-SVM & 0.941 & 0.32 & 0.478 \\
		%Abt-Buy               &  53,753 & 1075 & 50 & NN & 0.806 & 0.58 & 0.674 \\
		%Abt-Buy               &  53,753 & 1075 & 50 & AB & 0.722 &	0.52 & 0.605 \\
		%Abt-Buy               &  53,753 & 1075 & 50 & LR & 0.806 & 0.5 & 0.617 \\
		cora                  & 328,291 & 47.76 & 6874 & L-SVM & 0.841 & 0.837 & 0.839 \\
		$\star$ tweets100k    &  20,000 & 0.9903 & 10049 & L-SVM & 0.762 & 0.778 & 0.770 \\
		\hline
	\end{tabular}
	\label{tbl:pools}
\end{table*}

\begin{figure*}
	\centering
	\includegraphics[width=2.1in]{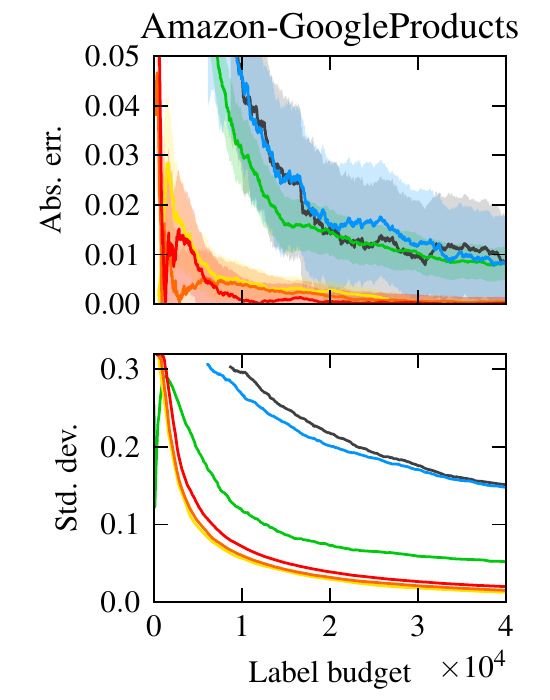}
	\hfill
	\includegraphics[width=2.1in]{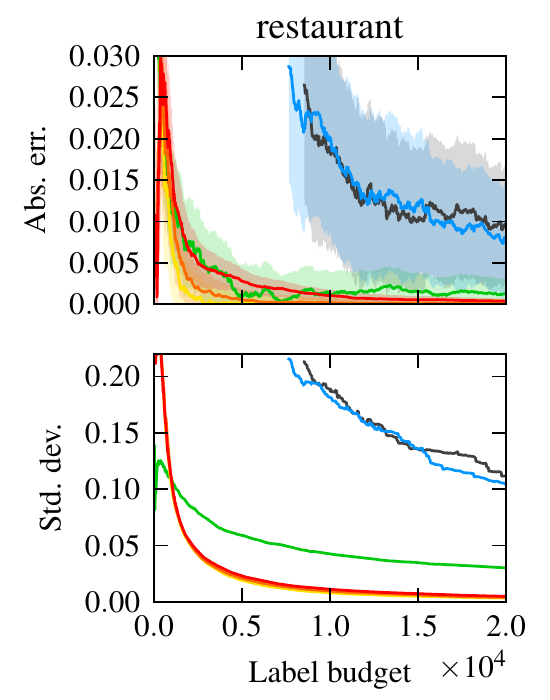}
	\hfill
	\includegraphics[width=2.1in]{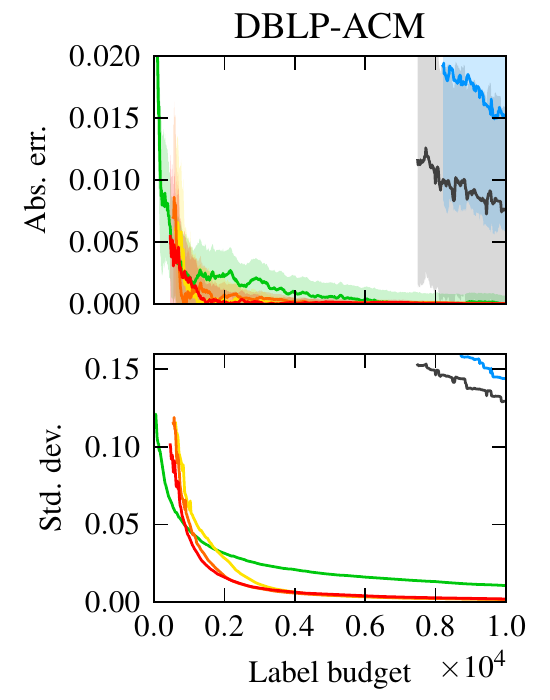}
	\\[1em]
	\includegraphics[width=2.1in]{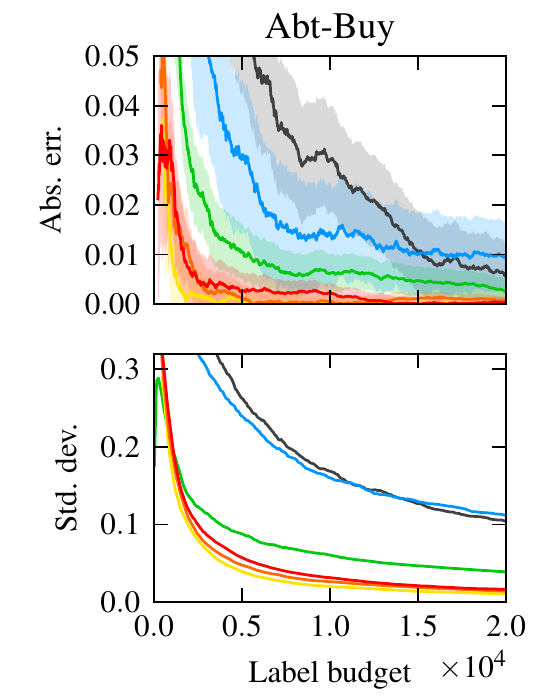}
	\hfill
	\includegraphics[width=2.1in]{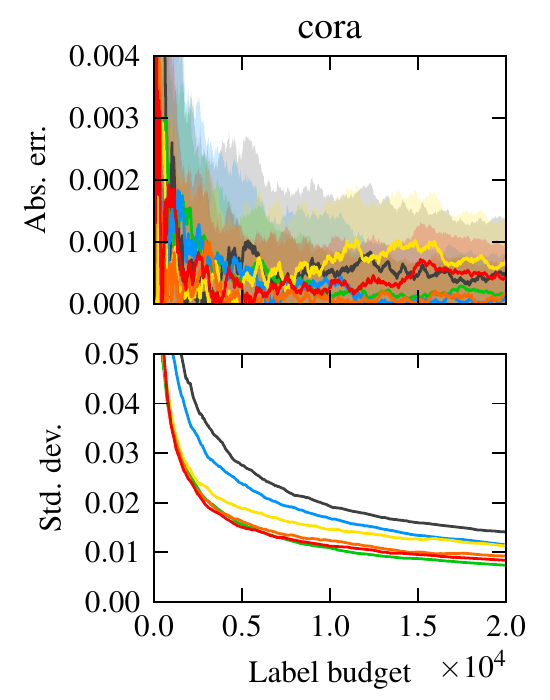}
	\hfill
	\includegraphics[width=2.1in]{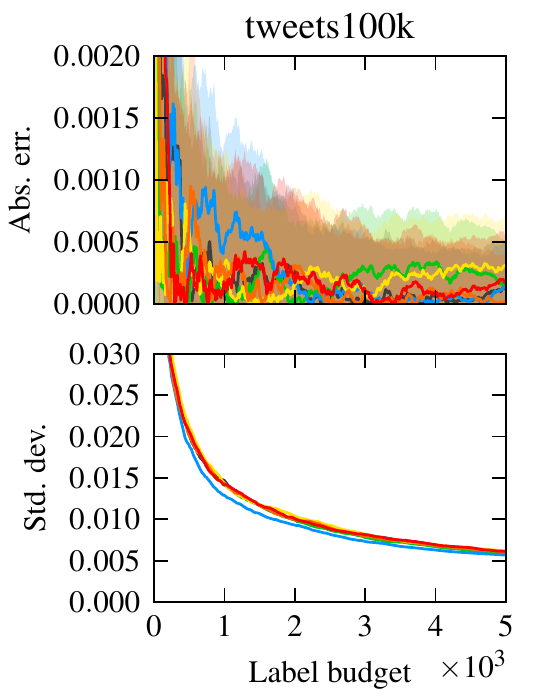}
	\\[2em]
	\includegraphics[width=5in]{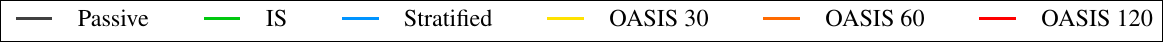}
	\caption{Plots showing the expected absolute error (abs.\ err.) and the standard deviation (std. dev.) of $\hat{F}_{1/2}$ for the different estimation methods (Passive, Stratified, IS, \alg) as a function of label budget. The \alg method is run with $K= 30, 60 \text{ and } 120$ (except on \texttt{tweets100k} where $K = 10, 20 \text{ and } 40$). This figure is best viewed in colour.}
	\label{fig:abs-err-and-std-dev-vs-labels}
\end{figure*}

%\begin{figure*}[t!]
%	\begin{minipage}[t]{0.5\textwidth}
%		\vspace{0pt}
%		\centering
%		\includegraphics[width=0.49\linewidth]{Abt-Buy_cal_vs_uncal_csf.pdf}
%		\includegraphics[width=0.49\linewidth]{DBLP-ACM_cal_vs_uncal_csf.pdf}
%		\\[1em]
%		\includegraphics[width=0.9\linewidth]{Abt-Buy_cal_vs_uncal_csf_legend.pdf}
%		\caption{Comparison of calibrated vs.\ uncalibrated scores for IS \& \alg (run with $K = 60$).}
%		\label{fig:calibrated-vs-uncalibrated}
%	\end{minipage} \hfill
%	\begin{minipage}[t]{0.46\textwidth}
%		\vspace{0pt}
%		\centering
%		\includegraphics[width=0.95\linewidth]{Abt-Buy_bar_plot_diff_class.pdf}
%		\caption{Expected absolute error in $\hat{F}_{1/2}$ for five classifiers trained on the Abt-Buy dataset. The error is measured after 5000 labels are consumed by each method (Passive, Stratified, IS, \alg). The error bars are approx.\ 95\% confidence intervals.}
%		\label{fig:bar-plot-compare-classifiers}	
%	\end{minipage}
%\end{figure*}

\subsection{Results}
Since each estimation method is randomised, we study their behaviour statistically. For each pool in Table~\ref{tbl:pools}, we run each estimation method 1000 times, recording the history of estimates for each run in a vector: $[\hat{F}_\alpha^{(t)}]_{t = 1:T}$. In all of the experiments, we set $\alpha = 1/2$, $\eta = 2K$ and $\varepsilon = 10^{-3}$.

\subsubsection{Label budget savings}
To compare the labelling requirements of the different estimation methods, we plot the expected absolute error $\operatorname{E}[|\hat{F}_\alpha - F_\alpha|]$ (abbreviated as abs.\ err.) as a function of the label budget.\footnote{Note that the label budget is not equivalent to the number of iterations. Since we are sampling with replacement, the same record pair may be drawn at multiple iterations, however it only counts towards the label budget the first (and only) time its label is queried from the oracle.} To compute abs.\ err.\ we average over 1000 repeats for fixed $P$. The true F-measure, $F_\alpha$, is calculated on $P$ using Eqn.~\eqref{eqn:F-measure-finite}, assuming all labels are known immediately. The results are presented in Figure~\ref{fig:abs-err-and-std-dev-vs-labels} for each pool in Table~\ref{tbl:pools}. Below the abs.\ err.\ plot, we have also plotted the standard deviation of the estimate, which is useful for checking whether the variance reduction methods (IS and \alg) are operating as designed.

\myparagraph{Winning method.} \alg beats the other methods, significantly improving on the state-of-the-art, both in terms of the abs.\ err.\ and the variance, on all of the ER datasets except \texttt{cora} where it is competitive. The reason for the anomalous behaviour on \texttt{cora} is likely due to the fact that the class imbalance is far less pronounced.

\myparagraph{Inadequacy of passive sampling.} The experiments confirm our claim that passive sampling is a poor choice for evaluating ER. Compared to IS and \alg, passive sampling demonstrates significantly slower convergence, and is less reliable due to the high variance. In fact, passive sampling often cannot produce any estimate at all until a significant label budget has been consumed (\cf \eg \texttt{DBLP-ACM}). This is because the F-measure remains undefined until a match (or predicted match) is sampled for the first time. We only begin plotting the curve when the estimate has a probability exceeding 95\% of being well-defined.

\myparagraph{Stratified method.}
This method does not fare much better than passive sampling, casting doubt on its
effectiveness for efficient evaluation as proposed in~\cite{druck_toward_2011}. We expect that the reason for the poor performance is due to the fact that the sampling is not biased (merely proportional to $\omega_k$).

\myparagraph{Balanced classes.} For the case of more balanced classes, as in \texttt{tweets100k}, and to a lesser extent \texttt{cora}, there is effectively no difference between the methods. This implies that the advantage of IS and \alg over the other methods diminishes as the imbalance ratio decreases. It is important to note however, that the balanced regime is of little relevance to ER---we merely include it for completeness. 

\subsubsection{Calibrated vs.\ uncalibrated scores}
In the experiments thus far (in Figure~\ref{fig:abs-err-and-std-dev-vs-labels}), we have been evaluating ER pipelines based on linear SVMs. The similarity scores from such systems are distances from the decision hyperplane, which are not intended to approximate the oracle probabilities $p(1|z)$ accurately (they are ``uncalibrated'' \cf Definition~\ref{def:calibration}). As such, we expect the performance of IS to be less favourable, because the instrumental distribution will be further from optimality if $s_i \approx p(1|z_i)$ is not satisfied. Much less degradation is expected under \alg.

In order to assess whether this has an appreciable effect, we compared running IS and \alg with calibrated versus uncalibrated similarity scores. The calibrated (probabilistic) scores are obtained using a built-in costly feature of LIBSVM, which runs five-fold cross-validation at training time~\cite{chang_libsvm_2011}. The uncalibrated scores are distances from the decision hyperplane used previously. The results in Figure~\ref{fig:calibrated-vs-uncalibrated} show that the calibrated scores yield significantly better performance, particularly for IS. However, the difference is less pronounced for \alg, which does a good job of learning the true oracle probabilities from the incoming labels.

\begin{figure}
	\centering
	\includegraphics[width=0.49\linewidth]{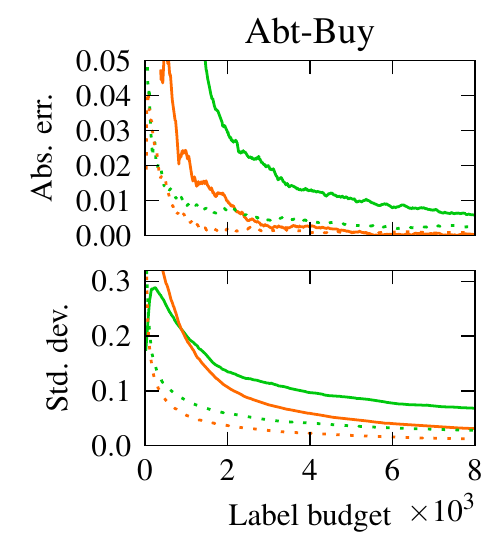}
	\includegraphics[width=0.49\linewidth]{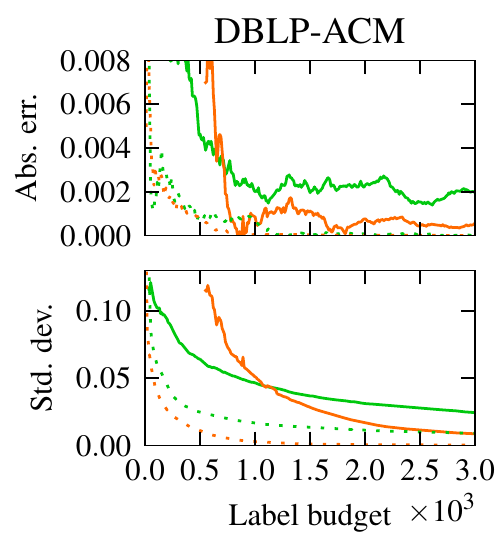}
	\\[1em]
	\includegraphics[width=0.9\linewidth]{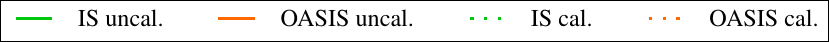}
	\caption{Comparison of calibrated vs.\ uncalibrated scores for IS \& \alg (run with $K = 60$).}
	\label{fig:calibrated-vs-uncalibrated}
\end{figure}

\subsubsection{Convergence of the model parameters}
We have observed excellent convergence properties for \alg in terms of the F-measure estimate. An interesting supplementary question is whether the estimates of the oracle probabilities (and in turn the instrumental distribution) also converge rapidly to their true (optimal) values. Although we have not studied this question theoretically, we have observed convergence in a limited number of experiments with \texttt{Abt-Buy}. An example is depicted in Figure~\ref{fig:convergence-inst-dist}. Heatmap plot~(b) demonstrates that the estimates of the oracle probabilities for this run converge quite rapidly: after $\sim4000$ labels are consumed. However, the instrumental distribution takes longer to converge, because it is very sensitive to slight errors in the estimates. It does not reach optimality until after $\sim8500$ labels are consumed. This is easiest to see in the KL divergence plot~(d), where a value of zero indicates convergence.

\begin{figure}[t]
	\centering
	\includegraphics[width=0.96\linewidth]{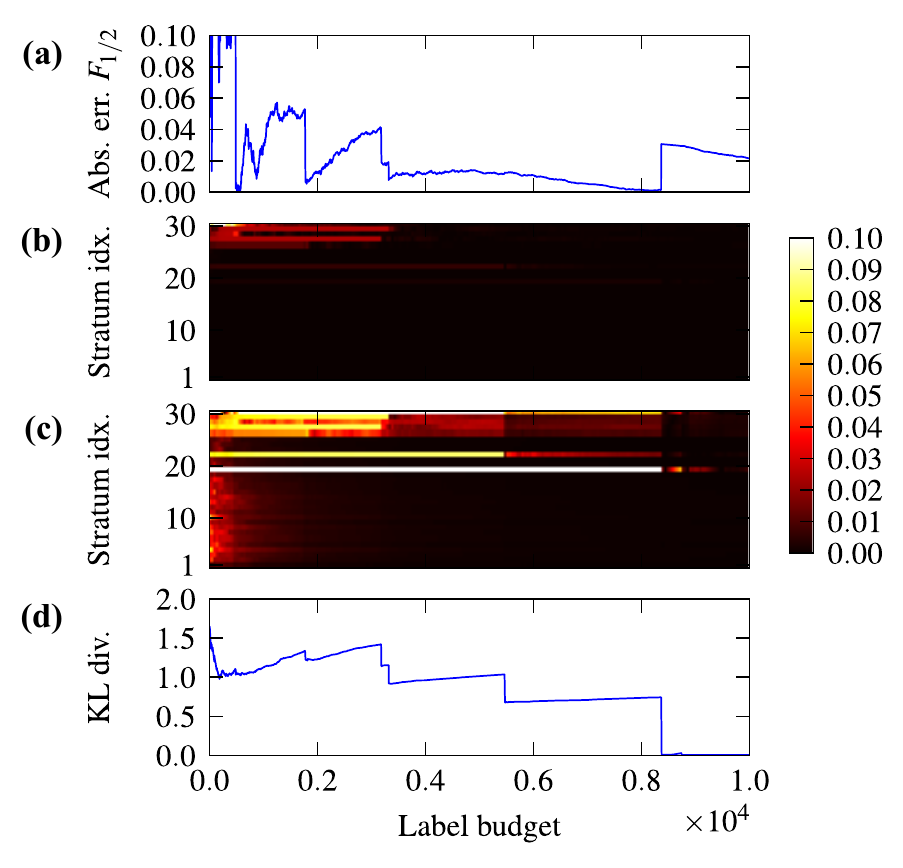}
	\caption{Convergence of the F-measure, oracle probabilities and instrumental distribution for a run of \alg on the Abt-Buy SVM dataset (with calibrated scores and $K = 30$): (a) absolute error in $F_{1/2}$; (b) absolute error in $\vec{\pi}$; (c) absolute error in $\vec{v}^{\star}$; (d) KL divergence from $\vec{v}^\star$ to the estimate $\vec{v}^{\star (t)}$.}
	%The estimate has converged when the KL divergence is zero.}
	\label{fig:convergence-inst-dist}
\end{figure}

\subsubsection{Effectiveness for different classifiers}
Although we have focussed on evaluating ER based on linear SVMs so far, there is essentially no limitation on the types of classifiers that can be evaluated, so long as they produce some kind of similarity scores. To this end, we re-run our experiments on the \texttt{Abt-Buy} pool using four additional types of classifiers: a neural network (multi-layer perceptron) with one hidden layer (NN), a boosted decision tree AdaBoost (AB), logistic regression (LR), and SVM with a RBF kernel (R-SVM). We implement the classifiers using scikit-learn with the default parameter options.

The expected estimation error for each method (Passive, Stratified, IS and \alg) is evaluated after 5000 labels are consumed and the results are plotted in Figure~\ref{fig:bar-plot-compare-classifiers}. We see that \alg generally outperforms the other methods, yielding an estimate of $F_{1/2}$ which is one order of magnitude more precise than IS. 
%These results provide evidence that OASIS is effective for a range of classifiers.

\begin{figure}
	\centering
	\includegraphics[width=0.99\linewidth]{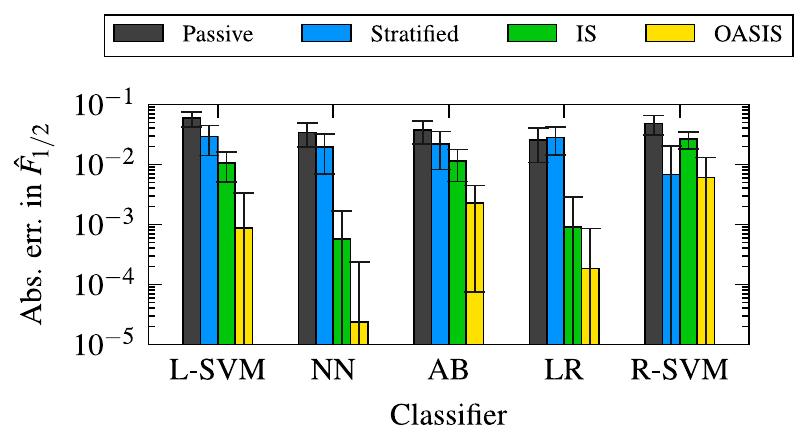}
	\caption{Expected absolute error in $\hat{F}_{1/2}$ for five classifiers trained on the Abt-Buy dataset. The error is measured after 5000 labels are consumed by each method (Passive, Stratified, IS, \alg). The error bars are approx.\ 95\% confidence intervals.}
	\label{fig:bar-plot-compare-classifiers}	
\end{figure}

\subsubsection{Runtime}
\label{sec:expt-timing}
We present evidence that the IS method scales poorly to large pools in Table~\ref{tbl:timings}, which lists the average CPU times for experiments on the \texttt{cora} dataset (pool size $N \sim 10^5$). The experiments were run on an HP EliteBook 840 G2 with 2.6GHz Core~i7 and 16GB RAM. Note that the times listed for the OASIS and Stratified methods exclude pre-computation of the strata, which takes less than 0.1~s. Looking at the results, we see that IS is an order of magnitude slower than \alg---in fact, the timing for IS appears to scale linearly in $N$ based on other timing data (not shown due to space constraints). The reason for this, is that IS samples from a non-uniform distribution over the entire pool (a computation linear in size $N$), whilst OASIS samples from a smaller non-uniform distribution over the strata (of size $K$). It appears that the extra operations OASIS requires to update the model are negligible in comparison.

\begin{table}
	\centering
	\caption{CPU times for the \texttt{cora} experiment.}
	\begin{tabular}{lrr}
		\hline
		Sampling method & \multicolumn{1}{p{2.2cm}}{Avg.\ CPU time per run (s)} & \multicolumn{1}{p{2.2cm}}{Avg.\ CPU time per iteration (s)} \\
		\hline
		Passive    &  0.512 & $2.483 \times 10^{-5}$ \\
		IS         & 69.854 & $3.149 \times 10^{-3}$ \\
		\alg 30   &  3.612 & $1.228 \times 10^{-4}$ \\
		\alg 60   &  3.281 & $1.123 \times 10^{-4}$ \\
		\alg 120  &  2.978 & $1.093 \times 10^{-4}$ \\
		Stratified &  1.967 & $9.502 \times 10^{-5}$ \\
		\hline
	\end{tabular}
	\label{tbl:timings}
\end{table}

\section{Related work}
\label{sec:related}
\myparagraph{Efficient evaluation.}
Previous work has considered efficient evaluation for general classifiers, through approaches such as importance sampling~\cite{sawade_active_2010},
stratified sampling~\cite{bennett_online_2010,druck_toward_2011} and semi-supervised inference of Bayesian generative models~\cite{welinder_lazy_2013}. However, none of this work accounts for the specific features of ER evaluation, namely extreme class imbalance, and the availability of auxiliary information in the form of similarity scores. 

Bennett \& Carvalho~\cite{bennett_online_2010} outline an adaptive method for estimating precision that stratifies
on classifier scores, sampling points with probability proportional
to the stratum population and a dynamic estimate of the variance in the labels. However, their method does
not incorporate recall and is not proven to be optimal. %Interesting feature: allow user to specify significance level and interval width.
Druck \& McCallum~\cite{druck_toward_2011} extend the work of \cite{bennett_online_2010} to facilitate estimation of vector-valued and non-linear functions
(including token-based accuracy and F-measure). Both of these approaches are adaptive and biased, although they rely purely on stratified sampling, which is known to be less effective at variance minimisation than importance sampling~\cite{rubinstein_simulation_2007}. We also note an exception in \cite{druck_toward_2011}: the method proposed specifically for estimating F-measure is based on proportional stratified sampling, which is neither adaptive nor biased.

Welinder \etal \cite{welinder_lazy_2013}
propose an estimation procedure for precision-recall curves, based on a Bayesian generative model. Their method is semi-supervised and makes use of the classifier scores,
but it doesn't incorporate biased sampling or adaptivity, making it unsuited to problems with class imbalance. It also imposes a restrictive assumption on the joint
distribution of scores and labels, requiring
the user to guess an appropriate parametric distribution.  Another non-adaptive approach
is the IS method of Sawade \etal~\cite{sawade_active_2010}. It facilitates the estimation of F-measures, relying on the asymptotically optimal distribution of equation~\eqref{eqn:optimal-inst-dist}. The authors address the instrumental distribution's dependence
on unknown quantities by estimating them using classifier scores. However if the scores are inaccurate or merely uncalibrated, the method
will be sub-optimal as it does not actively adapt using incoming labels.

\myparagraph{Adaptive importance sampling (AIS).}
A broad literature covers AIS, however to our knowledge, no prior work specialises these techniques to evaluation. A significant drawback of previous AIS algorithms, is that they discard and resample at each iteration, which is prohibitively wasteful when performing efficient evaluation. One of the earliest AIS algorithms is Population Monte Carlo (PMC), which maintains an entire \emph{population} of instrumental
distributions, updating them using propagation and resampling
steps~\cite{cappe_population_2004}. Standard formulations of PMC
use only the previous sample when updating
distributions, reducing statistical efficiency. Previous proofs of consistency also assume that the population
grows to an infinite size~\cite{douc_convergence_2007,cappe_adaptive_2008}. A more recent AIS algorithm is Adaptive Multiple Importance Sampling (AMIS) which is ``aimed at an optimal recycling of
past simulations in an iterated importance sampling (IS)
scheme''~\cite{cornuet_adaptive_2012}. Unlike PMC, AMIS makes use of the entire history of samples and instrumental distributions, to update the importance weights and instrumental distribution. However, it is not applicable in the efficient evaluation context because it requires an increasing sample to be drawn at each iteration, which would consume realistic label budgets too quickly.

%\begin{figure}[t]
%	\centering
%	\includegraphics[width=\linewidth]{Abt-Buy_inst_dist_convergence.pdf}
%	\caption{Convergence of the F-measure, oracle probabilities and instrumental distribution for a run of \alg on the Abt-Buy SVM dataset (with calibrated scores and $K = 30$): (a) absolute error in $F_{1/2}$; (b) absolute error in $\vec{\pi}$; (c) absolute error in $\vec{v}^{\star}$; (d) KL divergence from $\vec{v}^\star$ to the estimate $\vec{v}^{\star (t)}$.}
%	%The estimate has converged when the KL divergence is zero.}
%	\label{fig:convergence-inst-dist}
%\end{figure}

%\begin{table}[t!]
%	\centering
%	\begin{tabular}{lrr}
%		\hline
%		Sampling method & \multicolumn{1}{p{2.2cm}}{Avg.\ CPU time per run (s)} & \multicolumn{1}{p{2.2cm}}{Avg.\ CPU time per iteration (s)} \\
%		\hline
%		Passive    &  0.512 & $2.483 \times 10^{-5}$ \\
%		IS         & 69.854 & $3.149 \times 10^{-3}$ \\
%		\alg 30   &  3.612 & $1.228 \times 10^{-4}$ \\
%		\alg 60   &  3.281 & $1.123 \times 10^{-4}$ \\
%		\alg 120  &  2.978 & $1.093 \times 10^{-4}$ \\
%		Stratified &  1.967 & $9.502 \times 10^{-5}$ \\
%		\hline
%	\end{tabular}
%	\caption{CPU times for the \texttt{cora} experiment.}
%	\label{tbl:timings}
%\end{table}

\section{Conclusions}
We have proposed a novel adaptive importance sampler \alg for estimating 
the F-measure of ER pipelines. We leverage ER similarity scores through a
stratified Bayesian generative model, to update an instrumental sampling distribution
that optimises asymptotic variance. Statistical consistency establishes correctness of \alg,
while extensive experimentation demonstrates significant reduction to label budget relative to
existing approaches.

%We have proposed a novel AIS-based algorithm called OASIS, which solves the efficient evaluation problem for ER. Through comprehensive experiments, we have demonstrated that OASIS often exceeds the performance of competing approaches, yielding more precise estimates of the F-measure, with a significantly smaller label budget.

% Future work (if there's room)
%\textbf{Non-parametric density estimation.}
%Estimate $p(\vec{z}, \ell)$ using a non-parametric method, rather than our stratified approach.
%
%\textbf{Noisy oracle.}
%What if labels are noisy? E.g.\ implement the oracle using a crowdsourced labelling platform such as Amazon Mechanical Turk (AMT). Does everything still work?
%
%\textbf{Adaptivity and crowdsourcing platforms.}
%At the time of writing, as far as we are aware, crowdsourcing platforms don't support dynamic updating of the labelling task. As a result, it may be necessary to ask for the labels in batches, updating the instrumental distribution only once an entire batch has been labelled. Related: \cite{laws_active_2011} design a custom interface with AMT to support active learning.
%
%Label latency

% ensure same length columns on last page (might need two sub-sequent latex runs)
\balance

%ACKNOWLEDGMENTS are optional
\section{Acknowledgements}
N.\ Marchant acknowledges the support of an Australian Government Research Training Program Scholarship. B.\ Rubinstein acknowledges the support of the Australian Research Council (DP150103710).

\bibliographystyle{abbrv}

\begingroup
\raggedright
\small
\bibliography{sampling}
\endgroup
%APPENDIX is optional.
% ****************** APPENDIX **************************************
% Example of an appendix; typically would start on a new page
%pagebreak

\ifdefined\ARXIV
\appendix

\section*{Proof of Lemma~\ref{lem:LLN}}
The proof relies on the following theorem due to Petrov~\cite{petrov_strong_2014}.
\renewcommand\thetheorem{\unskip}
\begin{theorem}[Petrov strong LLN]
	Let $\{X_t\}_{t = 1}^{\infty}$ be a sequence of random variables with finite absolute moments of order $p > 1$ and zero means. Let $S_T = \sum_{t = 1}^{T} X_t$ for $T \geq 1$ and $S_0 = 0$. Assume the following condition is satisfied:
	\begin{equation*}
	\operatorname{E}|S_T - S_V|^p \leq A(T - V)^{p r - 1}
	\end{equation*}
	for all $T$ and $V$ such that $T > V \geq 0$, where $r \geq 1$ and $A$ is a constant. Then	$S_T/T^r \to 0 \ \text{a.s.}$
\end{theorem}

Let $X_t = U_t - \operatorname{E}[U_t]$. Observe that for $t > s \geq 1$,
\begin{align*}
\operatorname{E}[U_t U_s] &= \operatorname{E}\left[\operatorname{E}\left[U_t U_s | U_{t - 1}, \ldots, U_1\right]\right] \tag{tower law}\\
& = \operatorname{E}\left[U_s\operatorname{E}\left[U_t | U_{t - 1},  \ldots, U_1\right]\right] \\
& = \theta \operatorname{E}\left[U_s\right]= \theta^2,
\end{align*}
which implies that the sequence $\{U_t\}_{t = 1}^{\infty}$ is uncorrelated. It follows that for $t \neq s$,
\begin{equation*}
\operatorname{E}[X_t X_s] = \operatorname{E}[U_t U_s] - \operatorname{E}[U_t] \operatorname{E}[U_s] = 0.
\end{equation*}
Denoting $S_T = \sum_{t = 1}^{T} X_t$, this means that for $T \geq V \geq 1$,
\begin{align*}
\operatorname{E}[S_T S_V] &= \operatorname{E} \! \left[\sum_{t  = 1}^{T} X_t \cdot \sum_{v  = 1}^{V} X_v\right] = \sum_{t  = 1}^{T} \sum_{v  = 1}^{V} \operatorname{E}[X_t X_v] = \sum_{v = 1}^{V} \operatorname{E}[X_v^2]
\end{align*}
(the cross-terms vanish). As a result, we have that
\begin{align*}
\operatorname{E}[(S_T - S_V)^2] & = \operatorname{E}[S_T^2] + \operatorname{E}[S_V^2] - 2 \operatorname{E}[S_T S_V] \\
& = \sum_{t = 1}^{T} \operatorname{E}[X_t^2] - \sum_{t = 1}^{V} \operatorname{E}[X_t^2] \\
& = \sum_{t = V + 1}^{T} \operatorname{E}[X_t^2] \\
& = \sum_{t = V + 1}^{T} (\operatorname{E}[U_t^2] - \theta^2)  \leq C (T - V).
\end{align*}
This implies that the conditions of Petrov's strong LLN are satisfied for $p = 2$ and $r = 1$. Hence we conclude that $S_T/T \to 0 \ \text{a.s.}$, or equivalently, $\frac{1}{T} \sum_{t = 1}^{T} U_t \to \theta \ \text{a.s.}$ \qed

\section*{Proof of Theorem~\ref{thm:consistency-alg}}
The proof involves verifying that the conditions of Theorem~\ref{thm:consistency-F-alg} hold for the particular instrumental distribution used in Algorithm~\ref{alg:AIS}. This instrumental distribution is defined in equation~\eqref{eqn:alg-epsilon-greedy-inst-dist} over the strata, but can be equivalently expressed over the pool as follows:
\begin{equation}
	q_t(z) = \left[\omega_k v_k^{(t)}\right]_{k = \kappa(z)}
\end{equation}
where $\kappa: \mathcal{Z} \to \{1, \ldots, K\}$ maps a record pair $z$ to the stratum (index) that contains it. Furthermore, in the context of Algorithm~\ref{alg:AIS} we note that $p(z) = 1/N$ (uniform over the pool).

Having made these observations, we proceed to verify the conditions of Theorem~\ref{thm:consistency-F-alg}, which amounts to verifying the conditions of Theorem~\ref{thm:consistency-linear-AIS} for $f_\mathrm{num}$ and $f_\mathrm{den}$ (defined in equation~\ref{eqn:definition-fnum-fden}).

Recall that the first condition of Theorem~\ref{thm:consistency-linear-AIS} is that $q_t(z) > 0$ whenever $f(x)p(x) \neq 0$. This is clearly satisfied for both $f_\mathrm{num}$ and $f_\mathrm{den}$, since
\begin{equation*}
	q_t(z) = \varepsilon \cdot p(z) + (1-\varepsilon) \cdot q_t^\star(z) \geq \varepsilon p(z) > 0.
\end{equation*}
The second condition of Theorem~\ref{thm:consistency-linear-AIS} requires bounded second moments. Observe that for $L\sim\oracle(Z)$ and $\hat{L}=\hat{\ell}(Z)$
\begin{align*}
f_\text{num}(X)^2 & = L \, \hat{L},\\
f_\text{den}(X)^2 & = \alpha^2 \hat{L} + 2 \alpha (1- \alpha) \hat{L} \, L + (1-\alpha)^2 L,
\end{align*}
and
\begin{equation*}
\frac{p(z)}{q_t(z)} = \frac{1}{\varepsilon + (1 - \varepsilon) [|P_k| v_k^{(t)}]_{k = \kappa(z)}} \leq \frac{1}{\varepsilon}.
\end{equation*}

These observations imply that
\begin{align*}
& \underset{\substack{X_t \sim p \\ \vec{X}_{1:t-1} \sim g}}{\operatorname{E}} \left[\frac{p(X_t)}{q_t(X_t|\vec{X}_{1:t-1})} f_\mathrm{num}(X_t)^2 \right] \\
& = \underset{\vec{X}_{1:t-1} \sim g}{\operatorname{E}}\!\left[\int \frac{p(z)}{q_t(z | \vec{X}_{1:t-1})} \, \hat{\ell}(z) \, p(z) \, p(1 | z) \, dz \right] \\
& \leq \int \frac{1}{\varepsilon} \hat{\ell}(z) \, p(z) \, p(1 | z) \, dz \leq \frac{1}{\varepsilon} < \infty
\end{align*}
and
\begin{align*}
& \underset{\substack{X_t \sim p \\ \vec{X}_{1:t-1} \sim g}}{\operatorname{E}} \left[\frac{p(X_t)}{q_t(X_t|\vec{X}_{1:t-1})} f_\mathrm{den}(X_t)^2 \right] \\
& = \underset{\vec{X}_{1:t-1} \sim g}{\operatorname{E}}\! \Bigg[\int \frac{p(z)}{q_t(z | \vec{X}_{1:t-1})} \hat{\ell}(z) \left(\alpha^2 + 2 \alpha (1- \alpha)  p(1 | z)\right) p(z) \, dz \\
& \mspace{100mu} {} + (1 - \alpha)^2 \int \frac{p(z)}{q_t(z | \vec{X}_{1:t-1})} p(z) p(1|z) \, dz \Bigg] \\
& \leq \int \frac{1}{\varepsilon} \hat{\ell}(z) \left(\alpha^2 + 2 \alpha (1- \alpha) p(1|z)\right) p(z) \, dz \\
& \mspace{100mu} {} + (1 - \alpha)^2 \int \frac{1}{\varepsilon} p(z) p(1|z) \, dz\\
& \leq \frac{\alpha^2 + 2 \alpha(1-\alpha) + (1-\alpha)^2}{\varepsilon} = \frac{1}{\varepsilon} < \infty.
\end{align*}
This confirms that the conditions of Theorem~\ref{thm:consistency-F-alg} hold, thus completing the proof. \qed

\fi

\end{document}